\documentclass[acmsmall,nonacm]{acmart}

\acmJournal{FAC}

\usepackage{amsthm}
\usepackage{amsmath}

\newtheorem{theorem}{Theorem}[section]

\usepackage[vlined,lined,linesnumbered,ruled ]{algorithm2e}
\let\oldnl\nl
\newcommand{\nlnonumber}{\renewcommand{\nl}{\let\nl\oldnl}}

\usepackage{booktabs}

\usepackage{url}

\newtheorem{definition}{Definition}

\usepackage[most]{tcolorbox}

\usepackage[colorinlistoftodos, textwidth=3.5cm]{todonotes} 
\setuptodonotes{inline} 
\usepackage{tabularray}

\begin{document}

\title{Real-Time Model Checking for Closed-Loop Robot Reactive Planning}

\author{Christopher Chandler}
\orcid{0000-0001-5025-1498}
\affiliation{%
\department{Department of Electronic and Electrical Engineering}
\institution{University of Strathclyde}
\city{Glasgow}
\state{Scotland}
\country{UK}}
\email{christopher.chandler@strath.ac.uk}

\author{Bernd Porr}
\orcid{0000-0001-8157-998X}
\affiliation{%
\department{School of Engineering}
\institution{University of Glasgow}
\city{Glasgow}
\state{Scotland}
\country{UK}}
\email{bernd.porr@glasgow.ac.uk}

\author{Giulia Lafratta}
\orcid{0009-0006-2072-338X}
\affiliation{%
\department{School of Engineering}
\institution{University of Glasgow}
\city{Glasgow}
\state{Scotland}
\country{UK}}
\email{g.lafratta.1@research.gla.ac.uk}

\author{Alice Miller}
\orcid{0000-0002-0941-1717}
\affiliation{%
\department{School of Computing Science}
\institution{University of Glasgow}
\city{Glasgow}
\state{Scotland}
\country{UK}}
\email{alice.miller@glasgow.ac.uk}



\begin{abstract}
Reactive obstacle avoidance methods often cause agents to become trapped in local minima, because they can often only reason one step ahead (i.e., the next action based on the current state).  In this paper, we use model checking to achieve reactive multi-step planning and obstacle avoidance on an autonomous robot. Our small, purpose-built model checking algorithm generates plans \emph{in situ} (within the robot's code) based on ``core'' knowledge and attention as found in biological agents. This is achieved in real-time using no pre-computed data on a low-powered device. Our approach is based on chaining temporary control systems that are spawned to counteract disturbances in the local environment which disrupt an autonomous agent from its preferred action (or \textit{resting state}). We mitigate state-space explosion by relying on temporary snapshots of the immediate environment, restricting the number of states. Multi-step planning using counter-examples generated by depth-first search and a negated LTL path property is applied to scenarios involving a cul-de-sac and a free-standing obstacle. Empirical results and informal proofs of two fundamental properties demonstrate the effectiveness of our approach for the creation of efficient multi-step plans for local obstacle avoidance. We significantly improve performance compared to a purely reactive agent that can only plan one step ahead. Our approach is an instructional case study for the development of safe and reliable navigation in the context of autonomous vehicles. We believe it also  has general  application in navigation for mission-critical mobile robots.  
\end{abstract}

\maketitle
%
\section{Introduction}

Biological organisms are embodied agents that can form complex plans and respond flexibly to unexpected events when navigating an environment.  This is achieved in real-time without prior experience of the environment in a closed-loop fashion based on sensory input.  
Like biological organisms, an embodied robotic agent has an egocentric perspective of its surroundings---observations are relative to a local frame of reference and dependent on in-built sensors.  Upon sensing an unexpected input (e.g., an observation representing an obstacle or ``disturbance''), the agent can respond by performing an action to change its state. This action changes the environment, which is sensed by the agent, and the loop repeats until the obstacle is no longer in view \cite{Braitenberg1986Vehicles:Psychology}.  The agent only requires knowledge of the relative position of the obstacle (i.e., disturbance) to perform the action. 

This scenario represents a simple reflex response to the environment. However, using such reflex responses in biologically inspired robots, such as Braitenberg vehicles \cite{Braitenberg1986Vehicles:Psychology}, or even in commercial robots such as cleaning robots \cite{Corke2022}, is a risky strategy.  The stimulus-response approach to obstacle avoidance can often lead a robot to get stuck in concave objects, such as corners and cul-de-sacs \cite{Gogoi2022} (as would very simple biological organisms). In mission-critical contexts, such as remote space exploration \cite{toupetMars2026}, this could lead to catastrophic failure, wasting time and resources. Complex biological organisms, however, are capable of more sophisticated behaviour, such as the prediction of many consecutive obstacles and the generation of plans to counteract them. 
To achieve this, it is necessary for a robotic agent to access and interpret distal sensor information. 
Indeed, there is evidence that in biological systems an innate ``core'' understanding of
 causality allows organisms to reason about their environment and organise their behaviour in accordance with predicted outcomes \cite{Spelke2007CoreKnowledge,LeCun2022AOpenReview}. 
In other words, biological organisms have a mental model of the world  that allows them to simulate actions and predict their consequences beyond the present moment.

One promising model of this process
is ``vicarious trial and error''~\cite{Tolman1939PredictionSow-bug}, which describes how rats continuously look back-and-forth in navigation tasks, analogous to the notion of forward simulation of future actions and their outcomes when navigating unfamiliar environments.  
It has also been argued that we may reconstruct a perceptual scene using simplified internal representations of objects and their relevant physical properties, a notion referred to as the ``intuitive physical reasoning'' hypothesis \cite{battaglia_2013, lake_building_2016}. These representations are approximate and oversimplified but rich enough to support  mental simulations of objects in the immediate future.
 Indeed, this was recently reproduced in \cite{lafratta_2025} on a  robot using a gaming engine, which enabled reactive reasoning about the outcomes of future closed-loop actions for an unknown scene based on forward simulation. 

In general, faithful simulation of world physics is challenging and infeasible in real-time environments. However, as demonstrated in \cite{lafratta_2025}, it may not be necessary to consider all physical details when planning multiple steps ahead.  Rather, we can construct "approximate and oversimplified" representations of the immediate environment and focus the attention of the robot only on what is relevant for navigation. Model checking \cite{baier_principles_2008} is a widely used technique for automatically verifying reactive systems. It is based on a precise mathematical and unambiguous model of possible system behaviour and often uses abstraction to simplify complex system models where verification is intractable. To verify that a model meets specified requirements (usually expressed in temporal logic), all possible executions are systematically checked. If an execution that fails the specification is found, the execution path that caused the violation is returned. Model checking has been successfully used to ensure the safety and reliability of safety-critical systems such as flight control \cite{Wang2018}, space-craft controllers \cite{Havelund2001-3}, and satellite positioning systems \cite{lu2015}. It has also been applied to many aspects of software verification, e.g., to industrial systems \cite{uppaal-application3} and operating systems \cite{android}. 

In the context of autonomous robots, most research in model checking has focused on \textit{simulating} robotic behaviour in an environment using various levels of abstraction, often to minimise state-space explosion \cite{Henzinger1997, Dennis2016, Liu2017, Ferrando2018, Pek2020,frgihoirmino2020,hamilton2022, Lin2022}. In some cases, model checking has even been applied to \textit{parts of real} autonomous systems \cite{cardoso2021}. It has, for example, been used to check if sensor readings have arrived on time \cite{Foughali2020}, to inspect whether messages between nodes in the Robot Operating System (ROS) have been corrupted \cite{Huang2014,Ferrando2020}, to verify robot trajectories are safe
\cite{Bouraine2012}, and to detect faults in low level robotic control by comparing expected versus actual system readings \cite{Bohrer2018}. Generating plans using model checking for real autonomous robots in real-time, however, has only been achieved in a limited number of cases, and even then mostly in structured environments \cite{lehmannneedles2021,weissmann2011}, with only a few notable exceptions focusing on unstructured environments (e.g., \cite{Pivtoraiko2009}).  

In this paper, we demonstrate the use of model checking to generate multi-step plans for obstacle avoidance in \textit{real-time} on a \textit{real} autonomous robot. We do this in a way that reflects the behaviour of biological organisms by generating sequences of closed-loop controllers, maintaining a strict attentional focus on obstacles in the environment, conceived as disturbances which need to be eliminated using closed-loop control.
Our approach allows an agent to reason flexibly about the immediate environment based on sensory input using an explainable method. In comparison, popular black-box AI methods for autonomous navigation (e.g. deep reinforcement learning) make rigid decisions based on generalisations  \cite{sun_motion_2021} insensitive to the local environment. In safety-critical applications, such as autonomous vehicles (AVs), this not only undermines the robustness and safety of AI decisions, but diminishes system trustworthiness due to the lack of transparency \cite{european_commission_joint_research_centre_trustworthy_2021}. Our method is intended as a general concept which can serve as a drop-in replacement for pure reactive obstacle avoidance, for example in cleaning robots or factory floor delivery robots.
The main limitation of our approach is that it focuses on robust local obstacle avoidance and is not goal-directed. In Section \ref{sect:related}, we discuss preliminary attempts to address goal-directed behaviour and how it could be achieved by integrating with hierarchical decomposition (HD) \cite{Cai_Yan_Shi_Zhang_Guo_2023, Triolo2020}. 


We initially focus our attention, however, on two fundamental challenges for using model checking for multi-step planning and obstacle avoidance, as a proof of concept. One challenge is combining discrete control commands and interaction with a continuous environment, which is often unstructured \cite{Luckcuck2019}, as autonomous robots are hybrid systems \cite{Ramadge1987}. 
Another challenge is online planning in real-time.  Standard model checkers are not designed for real-time tasks. 
For this reason models are typically verified offline prior to implementation \cite{DALZILIO2023, weissmann2011} for real-time applications.
However, achieving multi-step planning and obstacle avoidance as a sufficient replacement for reactive methods requires the use of model checkers that can meet hard real-time constraints \cite{alur1990model}.

This problem was encountered in \cite{pagojus_simulation_2021}  (extended in \cite{MILLER_memoisation_SCP_2025}), where SPIN \cite{holz2011} and a 3D Unity simulation of an autonomous vehicle were combined to identify paths violating temporal properties for the planning and execution of consecutive overtaking manoeuvres on a traffic-heavy road. The model checker received information from the (simulated) autonomous vehicle, updated its current model, derived a safe path, and then communicated the path back to the autonomous vehicle. Although \cite{pagojus_simulation_2021} was a useful proof of concept, the time delay due to both model compilation and communication between the model checker and game engine (approximately 3 seconds) was unacceptable, despite the fact that model verification itself only took around 20 milliseconds. For this reason, SPIN could not be implemented directly on a robot for real-time planning. 

In this paper, we extend the results in \cite{pagojus_simulation_2021, MILLER_memoisation_SCP_2025} and  address the reported compilation problem by using a stripped-down model checking algorithm built directly into the robot code.  We do this for two reasons.  
First, we want to ensure that our method is performant on low-powered hardware to reduce development costs and promote wide applicability. The Perseverance rover completed 90\% of a 32.1 kilometres journey on the Martian surface using autonomous navigation on a small 133-MHz RAD750 processor (as of October 2024) \cite{toupetMars2026}, and fully autonomous cars are notoriously expensive to manufacture, in part due to computationally demanding real-time workloads that require expensive accelerated hardware \cite{Lin2018, shi_computing_2021}. Incorporating a pre-built model checker, such as SPIN, would mean installing additional functionality that is superfluous to need, consuming limited compute resources. Second, integrating model checking directly into the robot code from scratch means that our model does not have to be compiled prior to each plan, improving real-time performance enough to make multi-step planning and obstacle avoidance feasible. 

Another novelty of our approach is that we use a negated LTL path property to produce plans as counter-examples, rather than follow the traditional model checking approach in which full state-space exploration is used to verify the correctness of a system. Indeed, this is necessary in LTL as, unlike branching-time logics, there is no quantifier for \textit{some path}.  However, with careful model construction and specification of the LTL property, it is possible to generate relevant counter-examples that can represent multi-step plans for local obstacle avoidance---assuming the environment is static, that error in actuation falls within acceptable tolerance levels, and that LiDAR data conforms to a bounded sensor noise model. To avoid large state-spaces, we model temporary snapshots of the robot's immediate environment using an abstraction with a strict attentional focus on obstacles, conceived as disturbances to be eliminated. We demonstrate the effectiveness of our LTL-based approach for real-time planning and obstacle avoidance in a case study.  

In summary, our key contributions are as follows:  

\begin{itemize}
    \item The development of a bespoke stripped-down model checking algorithm;
    \item The direct integration of the model checker into the robot code for resource efficiency;
    \item The avoidance of large state-spaces via temporary snapshots of the environment;
    \item The use of counter-examples derived from LTL properties to generate plans;
    \item The generation of multi-step plans for obstacle avoidance (unlike typical reactive agents);
    \item A discussion of how our approach can be combined with others for generalisability.
\end{itemize}

In Section \ref{sect:relwork}, we provide a discussion of related work. We motivate the limitations of closed-loop control for obstacle avoidance and formulate the reactive planning problem in Section \ref{sect:motivation}. We give details of the low-powered robotic platform used in this paper and provide preliminary concepts and definitions in Section \ref{sect:prelim}.  In Section \ref{sect:methodology}, we describe our methodological approach for real-time planning using model checking. We explain the design of our case study and evaluate our results in Section \ref{sect:eval}. In Section \ref{sect:related}, we provide a discussion, identify the limitations of our approach, and the ways in which they can be overcome. We conclude and comment on future work in Section \ref{sect:conc}. 

The initial findings on our method were reported in an 18-page workshop paper  \cite{Chandler_2023}. In this much extended paper, we have added an in-depth examination of related work, a more comprehensive explanation of the methodology (including three new algorithms), and a broader discussion section. We have also extended our experimental results section to include details on experimental design, a more detailed analysis of the performance of our approach, and a second case study. We have also added two appendices. The first of these describes additional experiments that show that our method is cross-platform.  The second contains informal proofs of two fundamental properties of our approach. Small sections of text and some figures appear as in \cite{Chandler_2023}, but the majority of the text is new (or expanded), and some figures have been enhanced or replaced. Many of the figures in this paper are completely new, including Figures \ref{fig:statediagram} to  \ref{fig:composite}, Figure \ref{fig:automaton} and Figures \ref{fig:figure9} to \ref{fig:playground2}.

\section{Related Work}\label{sect:relwork}
Reinforcement Learning (RL) is currently the most popular paradigm for navigation in commercial applications. 
An RL agent learns to navigate an environment by taking a certain action aimed at maximising an objective function representing future reward in response to a sensory state \citep{Sutton1999BetweenLearning}.
The reward function is updated at the end of each episode, or at each training step~\citep{Sutton1998ReinforcementIntroduction}, and the aim of the learning process is to \textit{approximate} this function so that it generalises to a different testing environment. These days,  functions related to future rewards are often modelled as deep neural networks \cite{Mnih2013PlayingLearning}. 
Like biological systems, RL is \textit{causal}, in the sense that it cannot know the reward associated with a certain state until this is retrieved; hence the agent learns from experience, through trial-and-error. Deep RL has found many applications in autonomous driving since its inception~\cite{Pomerleau1988ALVINN:NETWORK, GoodfellowGenerativeNets,Russell1998LearningAbstract}, and has led to the ability to perform end-to-end training on complex networks which govern the entire control process, from sensing to motor control~\cite{Osinski2020Simulation-BasedDriving}. The use of deep RL, however, has many drawbacks. First, RL can only look one state ahead---it cannot predict a reward beyond the next state~\cite{SuttonReinforcementProgress}. Second, using deep neural networks to learn function approximations is a data-intensive process, and there is no guarantee that the system will be able to take a safe action in response to an unseen state~\cite{NationalTransportSafetyBoard2018Collision18}. Third, due to the \textit{causality} of the RL paradigm, the agent is inherently slow to train and can require several million trial-and-error runs to achieve performance that would be satisfactory on the road. Last, the use of deep networks in RL slows down the training process significantly due to its reliance on the inefficient back-propagation algorithm~\cite{Linnainmaa1976TaylorError, Rumelhart1986LearningErrors}. 

The following algorithms, on the other hand, are more suitable for fully online robotic planning.
A* is a popular algorithm in robotic planning developed by~\citet{Hart1968}. The environment is abstracted into a discrete graph representation, where each node in a graph represents an approximate location, and is searched  using a cost function that is computed as the sum of past cost and heuristic cost. These cost measures define a \emph{priority of expansion} for each node---nodes with the highest priority are assumed to be more likely to lead to an optimal solution. The definitions of past cost and heuristic cost can differ depending on implementation, however they are typically related to the distance travelled (past cost) and the distance which the agent estimates is to be travelled (heuristic cost) from a particular node in the graph. A* is guaranteed to lead to the optimal, or near-optimal, solution~\citep{gass2001near}; it has also been adapted for local planning, featuring a very fine discretisation of the environment~\citep{Liu2022ASL-DWA:Robots, Zheng2023TheAvoidance}, and dynamic replanning~\citep{Stentz-1993-13555, Koenig2002DLite}. The discretisation, however fine, is a major drawback of this approach when it comes to executing in the real, continuous world. Hence plans may be suboptimal at the continuous level, or error due to the approximation of the robot's true location may accumulate. 
In very finely discretised spaces, the A* search occurs over a dense set of nodes~\citep{Montemerlo2009Junior:Challenge}, which demands a high number of computations even in confined spaces.  


The Rapidly-exploring Random Trees (RRT) algorithm avoids the need for discretisation by sampling the obstacle-free area of the environment randomly within a pre-set distance from an initial node.  If it is possible to connect the initial node with the sampled one (i.e., no obstacles exist between nodes), the latter is added to the graph, and the process repeats until the goal is reached~\citep{LaValle1998Rapidly-exploringPlanning}. This algorithm was further improved as RRT*~\citep{Karaman2011}, which removes redundant edges and rewires them to achieve minimum cost paths. This minimum cost is calculated at the discrete level, while at the continuous level, the algorithm can still lead to unnecessary jitter in the trajectory. 

Lattice-based planning is another sampling-based strategy~\cite{Pivtoraiko2009} which discretises the environment into a graph. Unlike in RRT, the environment is sampled deterministically: a number of out-edges are expanded from each state, where each edge represents a pre-computed control action. In contrast with grid-based approaches, using a control function allows non-linear functions to be introduced. This leads to smoother trajectories, reducing the need for later optimisation. Discrete search algorithms (e.g., A*) can then be applied to find a path to the goal. Lattice-based planning is generally more computationally expensive than grid-based approaches, however the computational cost tends to level off with increasingly complex tasks~\cite{Pivtoraiko2009}.
Control actions are typically defined as motion primitives that can be optimised~\cite{bergman_improved_2021} to ensure continuity between states. However, these planners are subject to the chosen \textit{a priori} discretisation criteria and do not handle noise well. 

Control Barrier Function (CBF)-based methods define a constraint as a safe set of states that can make up a safe trajectory~\cite{AmesControlApplications}. These can be learned using supervised learning~\citep{Robey2020LearningDemonstrations} or RL~\cite{Xu2022A, Taylor2019EpisodicSystems}. Other approaches to learning constraints for CBF involve modelling disturbances offline (e.g., ~\cite{dhar_robust_2024}), and producing controls that are able to reject them~\cite{Hamdipoor2023SafeFunctions}. Online approaches also exist~\cite{Agrawal2021SafeFunctions}. CBF has typically found applications in cruise control~\cite{Taylor2019EpisodicSystems, Agrawal2021SafeFunctions}, but path planning is also possible~\cite{Agrawal2021SafeFunctions}. As CBF synthesises a control function, parameters must be chosen carefully and tailored to the specific deployment scenario; however, CBF inherently tends to fail in the case of unreliable motion actuation, due to environmental noise which is not accounted for in function design.

Generally, to achieve fast replanning in the presence of temporary obstacles, global planning algorithms can be supplemented with local replanning methods, such as Artificial Potential Fields (APF) ~\citep{Amiryan2020AdaptivePath}. The APF~\citep{Khatib1985Real-timeRobots} method operates as a closed-loop controller, where the trajectory of the robot is guided by two potential functions: an attractive potential for targets and a repulsive potential for obstacles. This algorithm make it possible to navigate continuous environments and does not assume full observability of the environment. It is therefore suitable for local planning. The downside is that it tends to converge at the local minimum, and the feasibility of the plans computed depends on the control functions used. One way to prevent this is to combine APF with global path planning algorithms \cite{Amiryan2020AdaptivePath,Abdel-Rahman2023ENHANCEDA-STAR}, which compute a globally optimal plan ``skeleton'' to define boundaries upon which the local algorithm optimises. Alternatively, deep neural networks can be used to adaptively learn optimal and nonlinear control functions ~\cite{Hua2025SynthesizingLearning, Li2025AFields} which are suitable for local planning and do not require a global environment map. On the other hand, deep neural networks require prior training on large datasets \textit{offline} and cannot be tuned at runtime.

Model Predictive Control (MPC) ~\cite{Richalet1978ModelProcesses, Cutler1979DynamicAlgorithm} is an optimisation-based method that was originally developed for industrial applications. It uses a model of the plant to predict its future behaviour and perform iterative optimisation over a receding time horizon. MPC is suitable for local (e.g., \cite{Soitinaho2023LocalControl}) and global (e.g., \cite{Liu2016AVehicles}) environment descriptions. Like Artificial Potential Fields, MPC has a tendency to converge to the local minimum, and requires application-specific hyper-parameter tuning, especially regarding the number of iterations to perform for optimisation~\cite{Zhang2023Collision-FreeControl, Jain2021OptimalControl}. 

\section{Motivating Example}\label{sect:motivation}

Figure \ref{fig:loop}A shows a robot driving in a straight line until it encounters an obstacle in the environment, represented by the disturbance $D$. Sensing the distrubance generates an error signal which triggers the agent to output a motor action  and change its state.  The motor action changes the state of the environment, which is in turn sensed by the robot and the loop repeats until $D$ is eliminated. 

\begin{figure}[!h]
    \centering
    \includegraphics[width=0.9\linewidth]{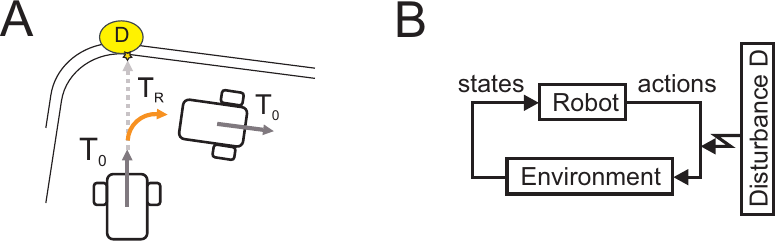}
    \Description{Figure 1.  Figure 1 A shows the robot driving straight in a task labelled T 0 and rotating right by executing task T R to eliminate a disturbance D. Figure 1 B is fully described in the text.}
    \caption{A:  robot encountering an obstacle in its environment, represented as the disturbance $D$.  B:  negative feedback loop for the elimination of disturbances.}
    \label{fig:loop}
\end{figure}

This simple reflex can be implemented using a closed-loop controller with a strict attentional focus on a single disturbance, as illustrated in Figure \ref{fig:loop}B.  We call such a controller a ``task''.  The disturbance $D$ represents an external obstacle which interrupts the  ``default'' action of the robot---driving  in a straight line, represented by the task $T_0$.   
The task $T_R$ is initiated when the disturbance $D$ comes into view, as shown in Figure \ref{fig:loop}A.  It has a finite lifetime.  Once the disturbance is out of sight, the task $T_R$ is no longer required, so it is terminated and the robot returns to the default task $T_0$. 

A task feedback loop is a simple mechanism where only information about the disturbance \textit{at the present moment} can be encoded or retained. This means that a task is (i) memoryless and (ii) produces stereotyped behaviour. It has no information on whether it will have to counteract future disturbances as a consequence of its actions. 
This can lead to inefficient agent behaviour. 

In Figures~\ref{fig:planning}A and \ref{fig:planning}B, for example, the robot starts by executing the default task $T_0$. Through distal sensors, the robot receives information that the space ahead contains the disturbance $D_1$.  Hence it must perform a control action to counteract $D_1$ by rotating either left or right, represented by tasks $T_L$ and $T_R$, respectively. From the perspective of the task, 
each rotation can achieve the control goal of eliminating the disturbance $D_1$, so there is no inherent preference.  

In Figure \ref{fig:planning}A, the robot initiates task $T_L$ to rotate left until the disturbance $D_1$ is out of view.  The robot can then return to the default task $T_0$ and safely continue to drive straight.  However, the robot could have equally chosen to initiate task $T_R$ and rotate right, as shown in Figure \ref{fig:planning}B.  In this case, it would immediately encounter the disturbance $D_2$ after returning to the default task $T_0$. Consequently, the robot must again perform an avoidance manoeuvre. If the robot then initiates task $T_R$, it would  escape the corner and continue in the default task $T_0$ undisturbed. If the robot initiates task $T_L$, however, yet another obstacle would be encountered. As there is no internal mechanism to prefer task $T_L$ or $T_R$, there is therefore a chance that the robot could oscillate repeatedly between avoid tasks and get trapped in the corner, which is unwanted behaviour.   

\begin{figure}[!t]
    \centering
    \includegraphics[width=0.9\linewidth]{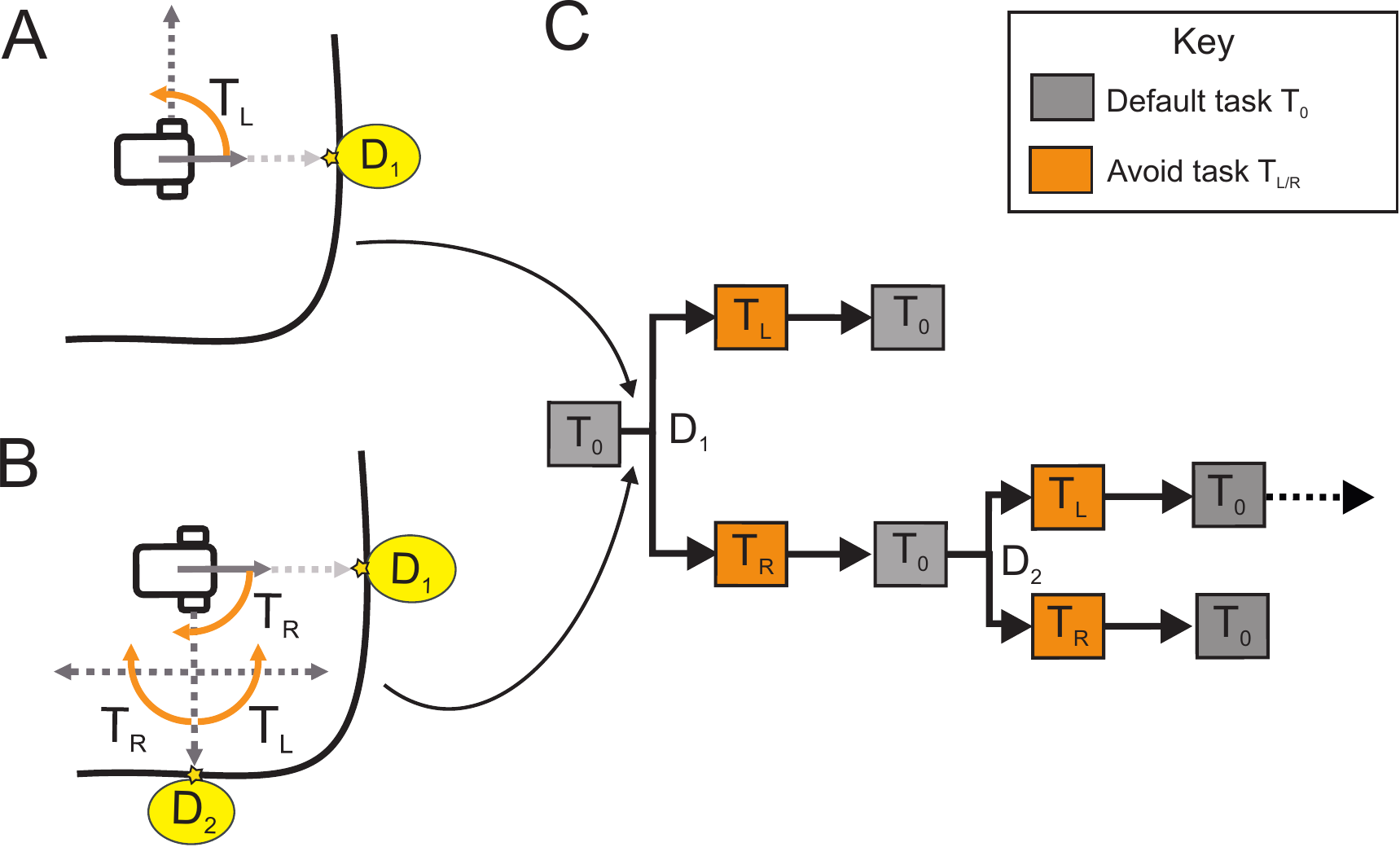}
    \Description{Fig.~\ref{fig:planning} A: and B: are fully described in the text. C: shows a behaviour tree linking possible sequences of tasks, some are the default task T0 while others are avoid tasks spawned to counteract disturbances.}
    \caption{Uncertainty in corner situation from the perspective of a closed task feedback loop.  A and B: valid options for eliminating the disturbance $D_1$ with varying consequences. C: behaviour tree for reasoning about multiple disturbances to form chains of closed-loop tasks, which requires extending the scope of environment observation. }
    \label{fig:planning}
\end{figure}

It is obvious to an external observer that the behaviour depicted in Figure~\ref{fig:planning}A is more efficient and therefore desirable. As 
a closed feedback loop can only encode and retain information about an immediate disturbance, however, choosing the most efficient option for obstacle avoidance is impossible. This requires extending the scope of reasoning beyond the present moment. 

In this paper, we use model checking to explore \textit{future} system states resulting from the execution of task sequences that satisfy an invariance constraint for obstacle avoidance.  Our aim is to reason about possible chains of closed-loop tasks, as illustrated by the behaviour tree shown in Figure \ref{fig:planning}C. 

\section{Preliminaries}\label{sect:prelim}

In this section, basic notions are introduced to motivate the development of our method.  We give details on our robotic research platform and sensing capabilities in Section \ref{sect:dims}, define the high-level architecture of our robotic agent in Section \ref{sect:arch}, and provide formal definitions in Section \ref{sect:formal}.  

\subsection{Robotic Platform}\label{sect:dims}

We developed our method on a low-powered robot shown in Figure \ref{fig:dims}.  Our robot was adapted from a widely available mobile robot development platform, AlphaBot by Waveshare\footnote{\url{https://www.waveshare.com/alphabot-robot.htm}}.  For sensing the environment, we equipped the robot with a low cost 360 degree 2D laser scanner, RPLiDAR A1M8 by Slamtec\footnote{\url{https://www.slamtec.com/en/LiDAR/A1/}}, and for actuation we used two continuous rotation servos by Parallax\footnote{\url{https://www.parallax.com/product/parallax-continuous-rotation-servo/}}.  The hardware programming interface for the robot was a Raspberry Pi 3 Model B\footnote{\url{https://www.raspberrypi.com/products/raspberry-pi-3-model-b/}} included with the AlphaBot development kit running a Quad Core 1.2GHz Broadcom 64bit CPU with 1GB RAM and wireless LAN. 

\begin{figure}[!h]
    \centering
    \includegraphics[width=0.9\linewidth]{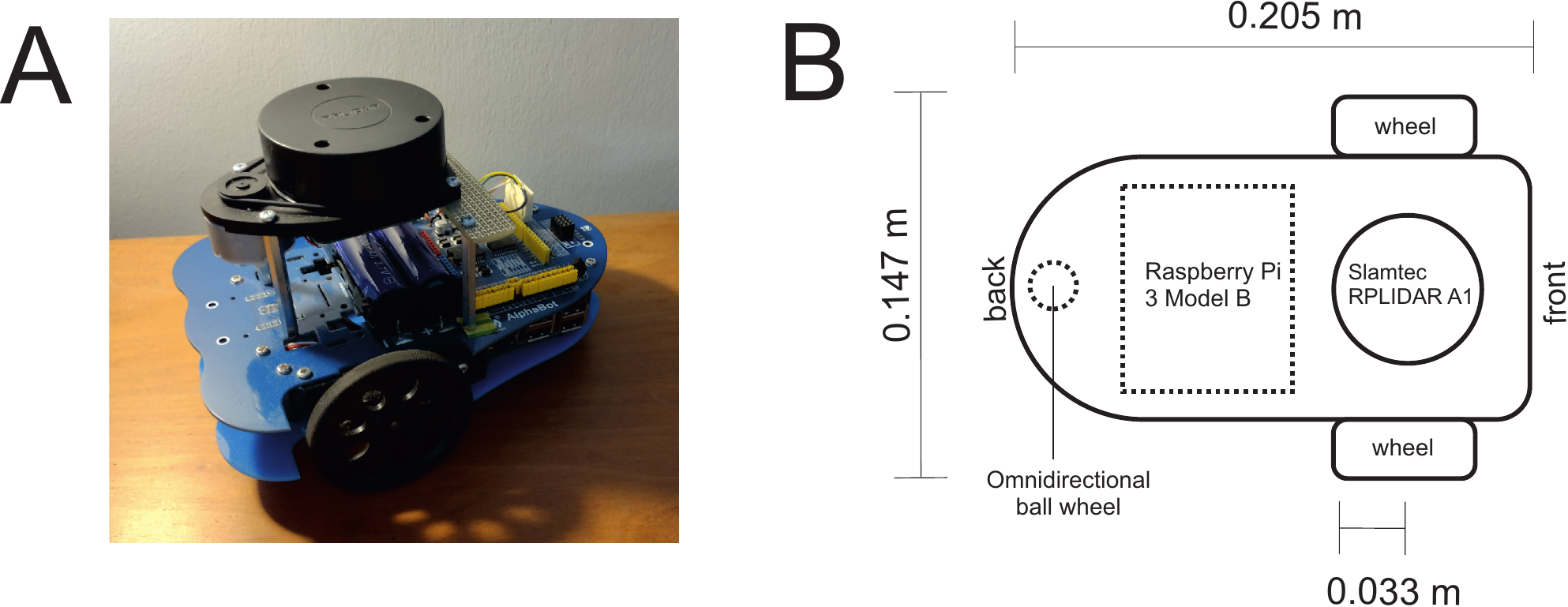}
    \Description{Fig.~\ref{fig:dims} A: is fully described in the text. B: shows a top-down schematic of the robot with dimensions.  Wheelbase is 0.147 metres, length is 0.205 metres, wheel radius is 0.033 metres.}
    \caption{A: our robotic platform.  B: schematic showing robot dimensions.}
    \label{fig:dims}
\end{figure}

The LiDAR generates a 2D point cloud at a rate of $\approx5~\mathrm{Hz}$.
We use a 2D vector space to model the point cloud with the origin representing a point at the centre of the LiDAR.  The robot is oriented towards positive $x$ by convention which corresponds to the heading $0$ radians.  Rotating $90$ degrees left is equivalent to the heading $\frac{1}{2}\pi$ radians, and rotating $90$ degrees right is equivalent to the heading $-\frac{1}{2}\pi$ radians. The maximum heading is $\pi$ in the left direction and $-\pi$ in the right, representing a $180$ degrees rotation relative to the local frame. A demonstration that the model checker is also able run on a different robot platform is included in Appendix \ref{app:cross}. The source code and instructions on how to run the model checker on both robotic platforms are available on Zenodo \cite{chandler_2026_21000117}. The Zenodo repository also includes analysis scripts, instructions on how to reproduce the results and a link to request the raw data (including videos of runs) collected for the case study.

\subsection{Agent Architecture}\label{sect:arch}

Figure~\ref{fig:seperate} shows the agent architecture. Our focus is on planning sequences of discrete closed-loop control tasks using in-built sensors to promote reactivity to imminent collisions.  
During execution, the robot reasons about current task outcomes.  When a disturbance is encountered, model checking is used to reason about the outcomes of future tasks and possible task sequences.  This separation of concerns between current and future task outcomes is reflected in the architecture shown in Figure \ref{fig:seperate}. Our robotic agent is implemented in C++ using an event-driven paradigm \cite{chandler_2026_21000117}.

\begin{figure}[!h]
    \centering
    \includegraphics[width=0.65\linewidth]{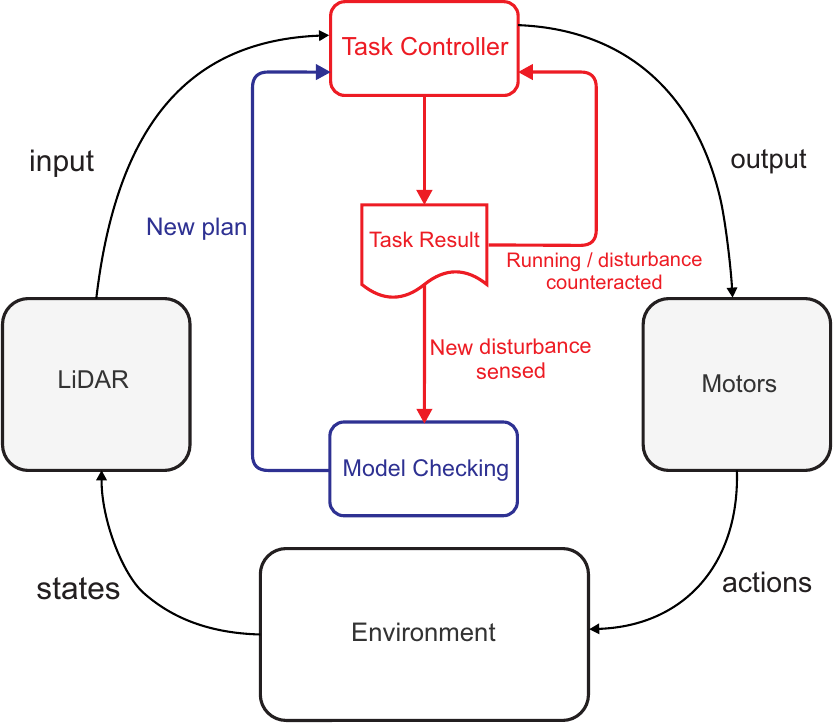}
    \Description{Fig~\ref{fig:seperate} Closed-loop agent architecture where states are sensed by the lidar and passed as input to the task controller.  The task controller sends a motor signal for the current task and returns a task result.  If the result is a new disturbance, then model checking is used to reason about future task outcomes and generate a plan, otherwise execution is returned to the task controller.  The motors output actions which changes the environment state and the loop repeats.}
    \caption{Agent architecture. Red indicates reasoning based on  current task outcomes.  Blue indicates planning using model checking based on future task outcomes.}
    \label{fig:seperate}
\end{figure}

\subsection{Model Checking as Planning}\label{sect:formal}

The model checking problem involves determining whether a finite-state model, describing the behaviour of a reactive system, satisfies a temporal logic
formula specifying a desired safety or liveness property of the system.  A \textit{Finite Transition System} is a common formalism for representing such a model, and temporal logic properties are usually expressed in a (sub-logic of) $CTL^*$ \cite{clarke1999}.

We use a simple model to represent the behaviour of a robot as it executes sequences of motion primitives (i.e., tasks) to avoid obstacles (i.e., disturbances) in the local environment. Let $\emptyset \not = Q \subset C$ denote a finite set of possible future configurations the robot could visit on the plane, relative to a local frame of reference, where $C$ denotes the set of all possible robot configurations, otherwise known as $C$-space \cite{siciliano_springer_2016}.  A Finite Transition System is defined in terms of $Q$ as follows:

\begin{definition}\label{fts}[Finite Transition System]
   A Finite Transition System $\mathcal{M}$ is a tuple \newline$\mathcal{M} = (S, Act, \rightarrow, I, \Phi, L)$ where

\begin{itemize}
    \item $S = Q$ is a non-empty, finite set of states;
    \item $\mathcal{T}$ is a set of tasks;
    \item $\rightarrow \subseteq S \times \mathcal{T} \times S$ is a transition relation;
    \item $I$ is a set of initial states; 
    \item $\Phi$ is a set of state formulae;
    \item $L : S \rightarrow 2^{\Phi}$ is a labelling function. 
\end{itemize}
\end{definition}

A path in $\mathcal{M}$ from a state $s \in S$ is a finite sequence of states $\pi = s_0, s_1, s_2,..., s_n$ where $s_0 = s$, such that for all $i > 0$, $(s_{i-1}, s_i) \in \rightarrow$. If such a path $\pi$ exists in $\mathcal{M}$, then we say that state $s_n$ is reachable from state $s_0$.
When checking properties of a Finite Transition System, we are interested only in its reachable states. To eliminate the possibility that the model checker might become stuck in an infinite cycle (i.e., checking along an infinite path), we construct a Finite Transition System as cycle-free \textit{a priori}.
Hence our model $\mathcal{M}$ is a tree structure and all possible future paths of the robot are finite. 

\begin{figure}[!h]
    \centering
    \includegraphics[width=0.50\linewidth]{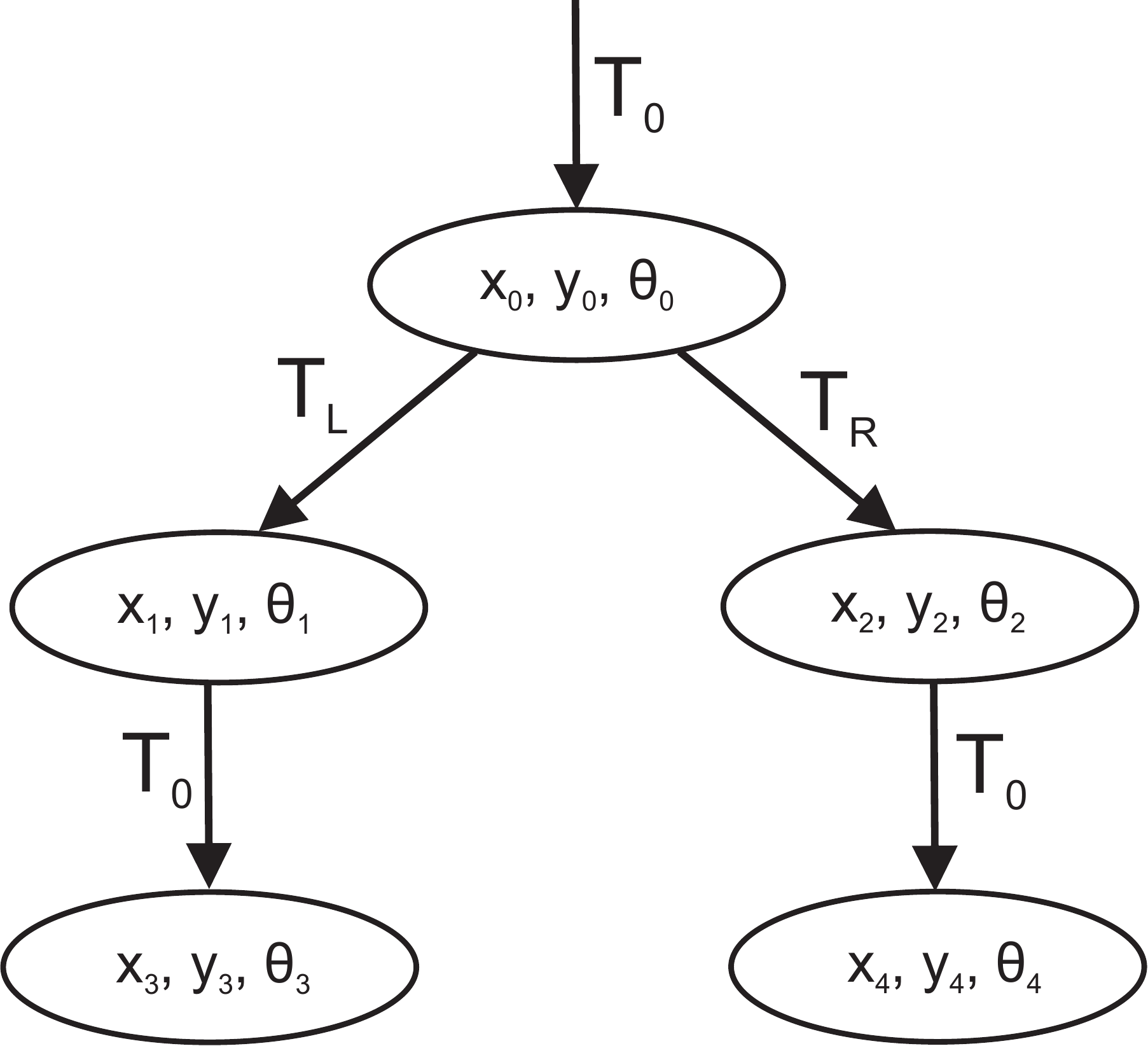}
    \Description{Fig.~\ref{fig:intro}. Example model $\mathcal{M}$ representing a robot which can only plan a single step ahead.}
    \caption{Example model $\mathcal{M}$ representing a robot which can only plan a single step ahead.}
    \label{fig:intro}
\end{figure}

Figure \ref{fig:intro} shows a model $\mathcal{M}$ representing a finite set of reachable states for a robot which can execute a single avoid task $T_{L/R}$ (rotating $90$ degrees left or right) in response to a disturbance, then return to the default task $T_0$ of driving in a straight line (see Section \ref{sect:motivation} for an illustration).  The model consists of two possible future sequences of tasks, represented by chained transitions to leaf states.  A valuation of variables $x$, $y$ and $\theta$ represents a possible configuration of the robot in $C$-space.
The finite set of configurations $Q \subset C$ of the robot in the model constitutes our state-space $S$.  

Note that there is a single initial state, as the model is only relevant when a disturbance has been encountered, otherwise the robot is in the default task $T_0$.  The executing task will always be $T_0$ when the model is utilised.  Leaf states represent a bound on returning to the default task which ensures a minimum disturbance-free distance so that the robot has time to re-plan if necessary. As the default task $T_0$ is indefinite, if this condition is satisfied, then we denote the $y$-coordinate of a leaf node as infinite by convention, as it is unknown whether there is a disturbance beyond this distance in the lateral direction.  Hence while a path to a leaf state is finite in practice, we suspend the transition artificially to represent the indefinite execution of task $T_0$---until a disturbance contradicts this assumption and the transition collapses on to the initial state. Consequently, we can unfold the infinite task tree shown in Figure \ref{fig:planning}C using a finite model of the immediate space around the robot. 

To express properties of a Finite Transition System it is typical to use the branching-time temporal logic $CTL^*$ or one of its sub-logics, such as the more restrictive Linear Temporal Logic (LTL).  The set of LTL state formulae $\Phi$ are defined inductively over atomic propositions $a \in AP$ using operators for \textit{negation}, \textit{conjunction} and \textit{disjunction}; the atomic propositions are derived from state variables, the parameters of our abstraction approach and the relative location of disturbances. We provide examples of relevant atomic propositions for our approach in Section \ref{sect:ltl}. For a state formula $\phi$, we write $\mathcal{M}, s \models \phi$ if $\phi$ holds at state $s$ of $\mathcal{M}$. A path formula $\varphi$ is defined inductively from state formulae.  In addition to negation, conjunction and disjunction, the symbols $\mathbf{X}$, $\mathbf{U}$, $\mathbf{F}$ and $\mathbf{G}$ represent the standard \textit{next-time}, \textit{strong until}, \textit{eventually} and \textit{always} operators.  For this paper it is unnecessary to formally present the syntax and semantics of LTL, which are detailed in \cite{clarke1999, baier_principles_2008}. 

Normally in model checking, we are interested in determining whether a path formula $\varphi$ holds for all paths in a finite-state model of a reactive system, i.e., $\mathcal{M} \models \varphi$ if $\varphi$ holds for all paths in $\mathcal{M}$.  However, we are interested in identifying one path in our model to extract a sequence of tasks for avoiding obstacles (i.e., disturbances).  Unlike the more expressive $CTL^*$, there is no quantifier in LTL for \textit{some path}, however we can define a path formula in LTL which generates relevant counter-examples.  For a path formula $\varphi$, we therefore write $\mathcal{M} \not \models \varphi$ if $\neg \varphi$ holds for at least one path in $\mathcal{M}$. Consequently, we denote the returned sequence of states $\pi = s_0, s_1, s_2,...,s_n$ a \textit{solution path}.  From the solution path $\pi$ we can then formulate a plan by extracting the sequence of transitions.  

To generate a solution path for our path formula $\varphi$, we construct a nondeterministic finite automaton (NFA) $A_{\neg\varphi} = (Q, \Sigma, \delta, Q_{0}, F)$ where $Q$ is a finite set of states, $\Sigma = 2^{\Phi}$ is a finite alphabet on the set of state formulae $\Phi$, $\delta : Q \times \Sigma \rightarrow 2^{Q}$ is a transition relation, $Q_{0} \subseteq Q$ is a set of initial states, and $F \subseteq Q$ is a set of accepting states \cite{baier_principles_2008}.  The NFA can accept all \emph{finite} paths that satisfy $\neg\varphi$.  Checking whether $\neg \varphi$ holds for at least one path in $\mathcal{M}$ is equivalent to checking that the automaton formed as the product of the set of finite paths and NFA $A_{\neg\varphi}$ has an accepting path. We  provide examples in Section \ref{sect:product}. Note that to check for infinite path violations of a property we would use a B{\"u}chi automaton \cite{vardi_automata} rather than an NFA, however that is not required in this paper.

\section{Methodology}\label{sect:methodology}


In this section, we explain the details of our approach for real-time planning and obstacle avoidance, using an abstraction that represents temporary snapshots of the local environment. Our method assumes that the robot operates in a static environment, there is no execution error beyond acceptable tolerance levels, and LiDAR data is conformant to a bounded sensor noise model. We define the action space and the set of closed-loop tasks for our modelling approach in Section \ref{sect:tasks}. 
In Section \ref{sect:sim}, we describe our method for abstract reasoning about closed-loop control tasks using spatial relationships.  We describe the transition system used in our approach in Section \ref{sect:trans}. In Section \ref{sect:planning}, we explain how we use an LTL path property to specify and generate plans at runtime.  

\subsection{Action Space}\label{sect:tasks}

We define the action space for our approach as the set of tasks:

\begin{equation}\label{eq:tasks}
    \mathcal{T} = \{T_0, T_S, T_L, T_R\}
\end{equation}

\noindent some of which were introduced in the motivating example discussed in Section \ref{sect:motivation}. Note that in the default task $T_0$ the robot has the expectation that it can continue driving in a straight line for a (potentially) infinite period of time---until a disturbance $D$ is encountered. Consequently, the robot has the expectation that the path ahead is finite, so the default task becomes the finite straight task $T_S$. This also holds for intermediary straight tasks in a plan because the robot has reasoned that it can only drive straight for a finite period. Hence our approach utilises two straight tasks. 

\subsection{Abstract Closed-Loop Sequences}\label{sect:sim} 

As explained in Section \ref{sect:motivation}, a task feedback loop has an attentional focus on an immediate disturbance, hence it cannot predict whether it will be necessary for a robot to counteract a \textit{future} disturbance as a consequence of its execution. To achieve this, a robotic agent requires the ability to access and interpret distal sensor information.  We abstract away the timed aspects of closed-loop tasks and represent their continuous evolution with spatial constructions, defined in relation to 2D LiDAR data. If the data is conformant to a bounded noise model and actuation error is tolerable, this makes it possible to reason about the consequences of task execution and strategies to counteract future disturbances in the local environment. We also require that the robot returns to the default task $T_0$ as soon as possible without executing consecutive avoid tasks, a bias which we operationalise by ensuring that the robot prefers selecting the shortest safe sequence meeting this requirement. 

In the remainder of this section, we explain how we reason about the safety of future task execution for two, three, and four-step sequences.  We do this in order of sequence length. First we reason about two-step sequences, as they return to the default task $T_0$ after executing a single avoid task $T_{L/R}$.  If this is not possible, we then move on to reason about three-step sequences.  This is a safety feature for situations where there is no lateral room for manoeuvre, however it involves executing two consecutive avoid tasks. If a three-step plan is deemed unnecessary, we then move on to reason about the safety of four-step sequences, which do not involve consecutive avoid tasks.

\subsubsection{Two-Step Sequences}\label{sect:twostep}

The shortest possible sequence of tasks is two-steps. 
This represents a situation where the robot encounters a disturbance, executes a single avoid task $T_{L/R}$ then returns to the default task $T_0$  because no future disturbance has been predicted as a consequence of driving straight in the lateral dimension, as illustrated by the example in Figure~\ref{fig:planning}A for a left turn.  However, before detailing our approach, it is necessary to explain first how disturbances provide success and failure conditions for the closed-loop execution of the straight tasks $T_0$ and $T_S$.

\begin{figure}[!h]
    \centering
    \includegraphics[width=0.6\linewidth]{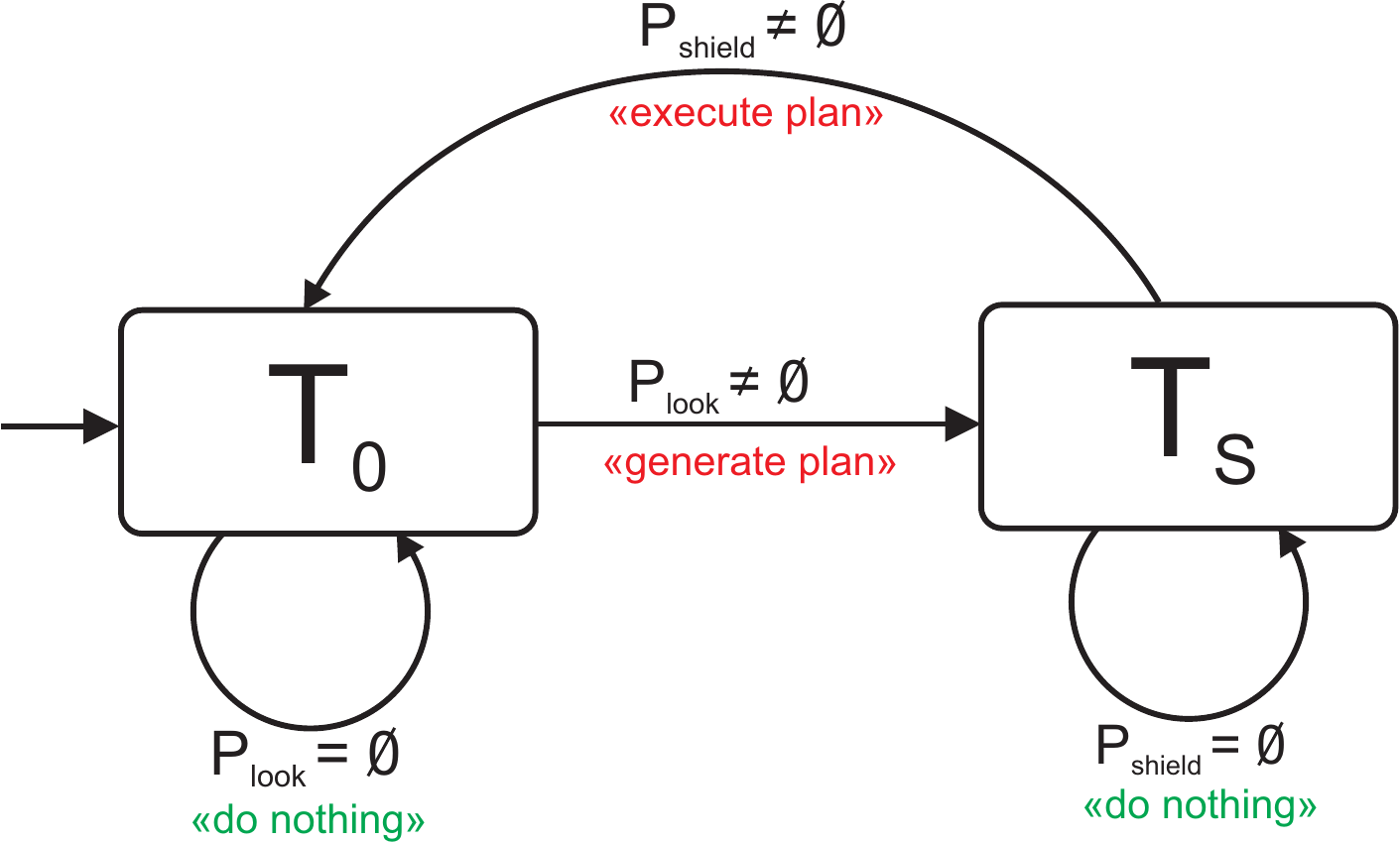}
    \Description{Fig~\ref{fig:statediagram}. Fully described in the text.}
    \caption{State diagram showing how $P_{look}$ and $P_{shield}$ provide success conditions for straight tasks.}
    \label{fig:statediagram}
\end{figure}

During each iteration of the control loop for a straight task (either the default task $T_0$ or the straight task $T_S$) we check whether there is a disturbance $D^+$ ahead, as this indicates that the robot can no longer continue driving straight. To do this we define two partitions in front of the robot with a specific attentional focus on checking for disturbances, the partitions $P_{look}$ and $P_{shield}$ shown in Figures~\ref{fig:onesuccess}A and \ref{fig:onesuccess}B, respectively. 
A state diagram illustrating this process is shown in Figure \ref{fig:statediagram}.
The robot starts in the default task $T_0$ and checks whether the $P_{look}$ partition is empty during each iteration of the control loop. If the partition $P_{look}$ is empty, then the robot can continue executing the default task $T_0$, hence during execution of the default task $P_{look} = \emptyset$ is a success condition for the task and $P_{look} \not = \emptyset$ is a failure condition which triggers the generation of a plan.  Immediately after a new plan is generated, the robot starts executing the straight task $T_S$, which has a finite lifetime as a disturbance is expected in the near future. If the partition $P_{shield}$ is empty, then the robot can continue executing the straight task $T_S$, hence during execution of the straight task $P_{shield} = \emptyset$ is a success condition for the task and $P_{shield} \not = \emptyset$ is a failure condition which triggers plan execution.




Building on the example in Figure \ref{fig:planning}A, we now explain reasoning about safety for two-step sequences, in this first case for a left turn where no second disturbance prevents the robot from returning to the default task $T_0$.  We provide a detailed example of this scenario in Figure \ref{fig:onesuccess}. In our example, the robot is initially 
driving straight in the default task $T_0$. The robot then witnesses a disturbance in the partition $P_{look}$ shown in Figure \ref{fig:onesuccess}A. As per the state diagram in Figure \ref{fig:statediagram}, this indicates that the default task $T_0$ is now a failure which triggers the generation of a new plan, which in our example is the sequence $T_L \rightarrow T_0$. 
As there is a disturbance straight ahead, it now has the expectation that the future path is finite, so continues driving in the straight task $T_S$ for a short while until the disturbance $D^+$ is witnessed in partition $P_{shield}$ (shown in Figure \ref{fig:onesuccess}B).  This indicates that $T_S$ is a failure which triggers the first task in the plan. In our example, this is the avoid task $T_L$. 

\begin{figure}[!t]
    \centering
    \includegraphics[width=\linewidth]{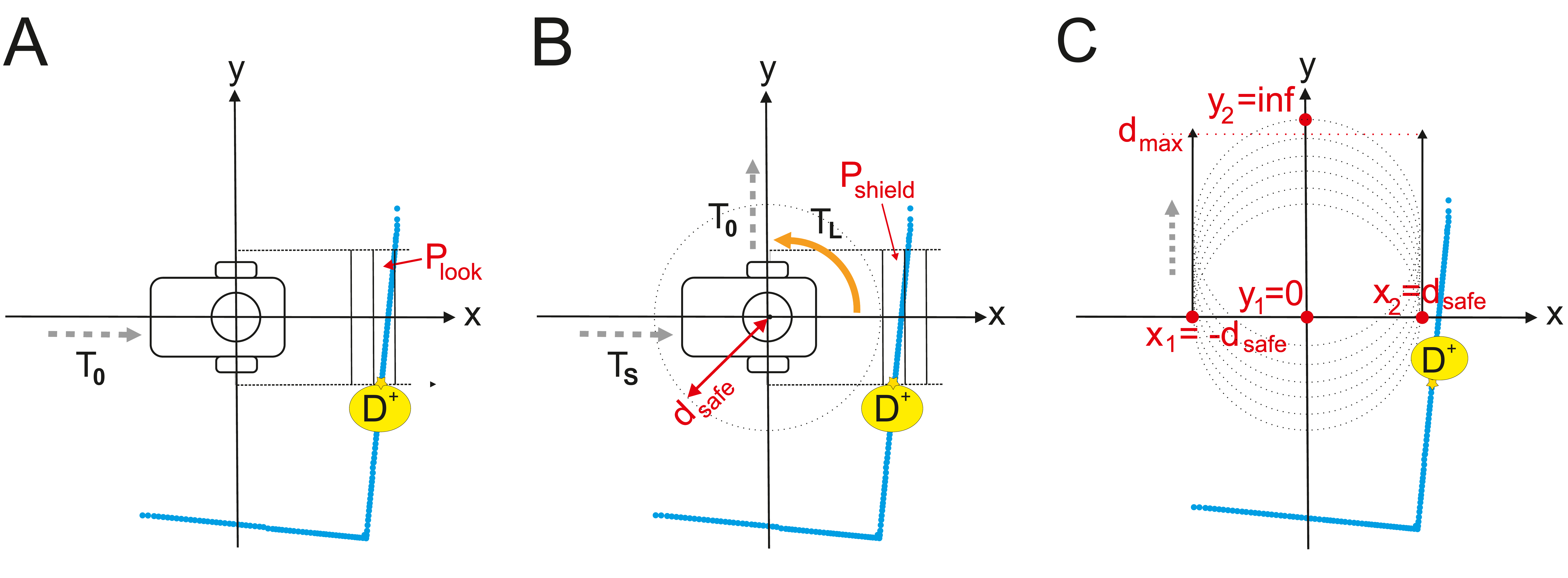}
    \Description{Fig.~\ref{fig:onesuccess}. Fully described in the text.}
    \caption{A: conceptual overview of the robot after witnessing a disturbance in $P_{look}$ triggering the generation of a two-step plan $T_L \rightarrow T_0$. B: after storing the plan and subsequently executing the straight task $T_S$ for a short while, execution of the plan is triggered when the disturbance is witnessed in $P_{shield}$. C: spatial abstraction used in our model for reasoning that about whether the robot can safely execute the plan $T_L \rightarrow T_0$.}
    \label{fig:onesuccess}
\end{figure}

To ensure an avoid task (i.e., rotation) is safe in our approach, we construct an egocentric safe zone around the robot (dotted circle in Figure \ref{fig:onesuccess}B) defined by the radius $d_{safe}$. No object is allowed inside the safe zone.  
During execution, we consider an avoid task a failure until the absolute difference between the current and initial angles of the disturbance is greater than $\frac{1}{2} \pi$. Once this condition has been satisfied, we denote the task a success which triggers the next task in the plan.

For reasoning about the safety of future straight tasks in the lateral dimension we construct lateral partitions of the state-space.  We iterate over the set of observations $O$ and subtract a longitudinal offset $\Delta^+ = D_x^+ - d_{safe}$ from the $x$-coordinate of each observation $o$ to simulate witnessing $D^+$ as a proximal disturbance in the future. In this paper, each observation $o$ represents a single point returned from a LiDAR scan.  Furthermore, we use the subscript $D_x^+$ to denote the $x$-coordinate of a disturbance and $D_y^+$ to denote the $y$-coordinate.  
We then construct a partition for the left direction:

\begin{equation}\label{eq:poslat}
    P_L = \{o \in O\ \boldsymbol{|}\ |o_x| \leq d_{safe} \wedge 0 < o_y \leq d_{max}\}
\end{equation}

\noindent where $d_{safe}$ is a bound on the longitudinal dimension 
($x$-axis) to respect the combined safe zone and outer shield $r + d_{shield}$, $0$ is a lower bound on the lateral dimension, and $d_{max}$ is an upper bound indicating a maximum lateral look ahead distance for witnessing distal disturbances.  The partition $P_L$ therefore includes all observations $o \in O$ to the left of the robot within the defined coordinate bounds. Note that the chosen bounds represent a rectangular area to the left which is an over-approximation of the robot's swept volume in a $\pm$ 90 degree rotation
; this is to focus attention on all obstacles which could result in a collision when driving straight after executing the avoid task $T_L$, whilst providing some room for  error.  However, we are only interested in the nearest observation to the left of robot for decision-making purposes, hence any disturbance to the left is defined as $D_L = min(o_y \in  P_L)$.  
As shown in Figure \ref{fig:onesuccess}B, the partition $P_L$ is empty, so in this case returning to the default task $T_0$ is possible.

\begin{figure}[!t]
    \centering
    \includegraphics[width=0.75\linewidth]{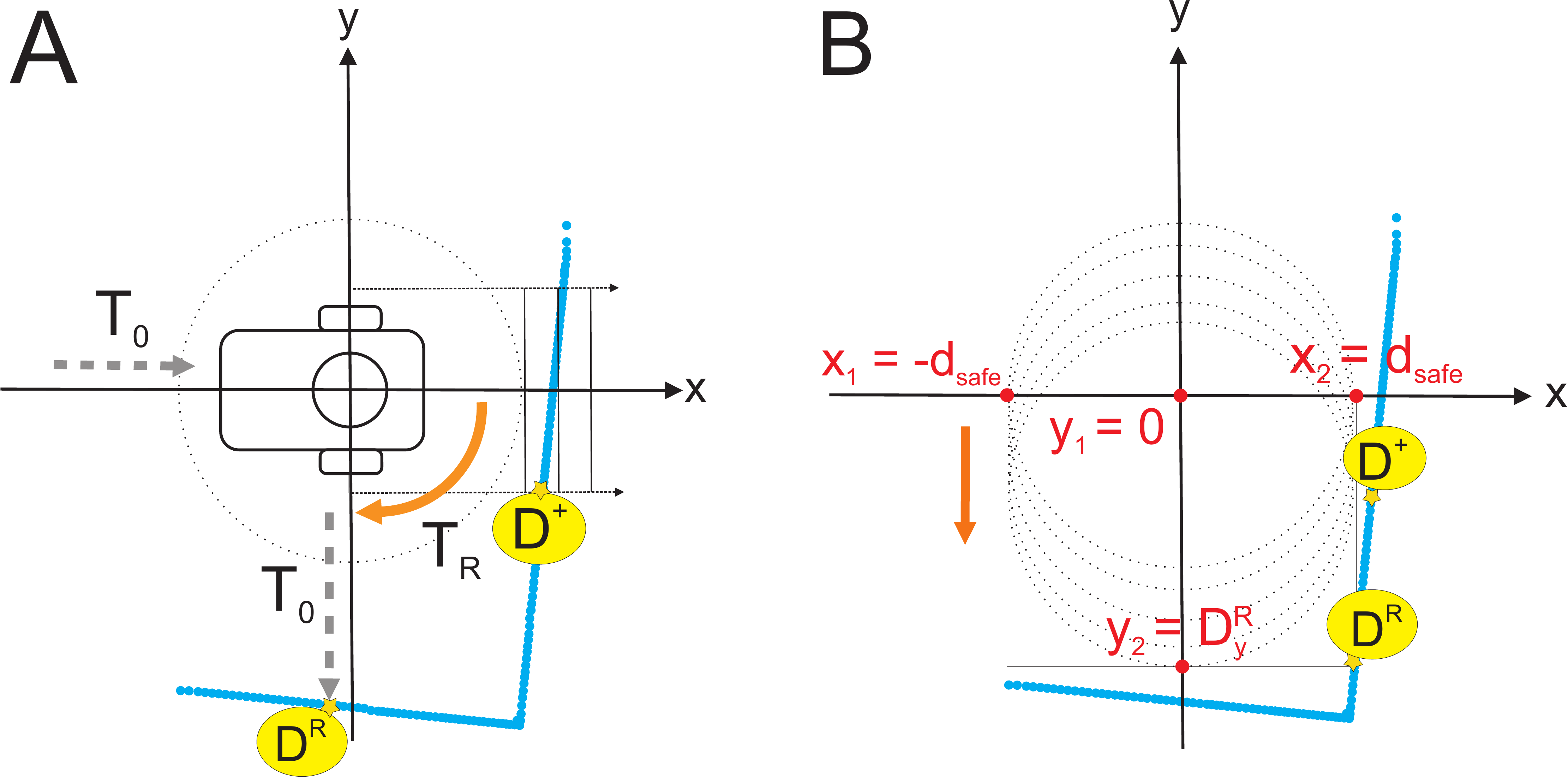}
    \Description{Fig.~\ref{fig:onefail}. Fully described in the text. }
    \caption{A: closed-loop execution of the two-step sequence $T_R \rightarrow T_0$. B: spatial abstraction of the closed-loop sequence in A which in this case witnesses disturbance $D_R$.}
    \label{fig:onefail}
\end{figure}

However, there could have been a second disturbance preventing the robot from returning to the default task $T_0$, as illustrated by the example shown in Figure~\ref{fig:planning}B for a right turn. 
Figure~\ref{fig:onefail}A shows a more detailed example of the robot executing a right turn in response to the disturbance $D^+$, i.e., the sequence of tasks $T_R \rightarrow T_0$. In this case, the robot encounters the disturbance $D^R$ in the right direction as a consequence of executing the avoid task $T_R$ which is again safe due to construction of the safe zone.  It is therefore only necessary to reason about the safety of driving straight in the lateral dimension. Assuming the real execution of tasks lies within acceptable tolerance levels and LiDAR data is conformant, the geometry of task execution is symmetrical to the previous sequence, so we can iterate over the set of observations and construct the negative reflection of Equation \ref{eq:poslat}:

\begin{equation}\label{eq:neglat}
    P_R = \{o \in O\ \boldsymbol{|}\ |o_x| \leq d_{safe} \wedge -d_{max} \leq o_y < 0\}
\end{equation}

\noindent where $d_{safe}$ is a bound on the longitudinal dimension respecting the combined safe zone and outer shield, $0$ is an upper bound on the lateral dimension, and $-d_{max}$ is a lower bound indicating a maximum look ahead distance for witnessing distal disturbances in the right direction. Similar to the left partition in Equation \ref{eq:poslat}, the partition $P_R$ includes all observations $o \in O$ to the right of the robot within the defined coordinate bounds, representing a rectangular over-approximation of the robot's swept volume in the right direction, in this case to focus attention on all obstacles to the right which could result in a collision when driving straight after executing the avoid task $T_R$, providing some room for error. Again we are only interested in the nearest observation to the robot for decision-making purposes. However as all $y$-coordinates to the right of the robot are negative, in this case a disturbance to the right is defined as $D_R = max(o_y \in P_R)$.  As shown in Figure \ref{fig:onefail}B, the partition $P_R$ is not empty, so in this example returning to the default task $T_0$ is not possible.

The possible outcomes for two-step sequences are summarised in Table \ref{tab:outcomes2}. The procedure for constructing lateral partitions is provided in Algorithm \ref{alg:lateralalg}.

\begin{table}[!h]
\caption{Possible outcomes for two-step sequences.}
\centering
\begin{tabular}{lllll}
\toprule
\textbf{Sequence} & \textbf{Success} & \textbf{Failure}  \\ \midrule
$T_L \rightarrow T_0$  & $P_L = \emptyset $ & $P_L \not = \emptyset $\\
$T_R \rightarrow T_0$  & $P_R = \emptyset $  & $P_R \not = \emptyset $     \\ \bottomrule 
\end{tabular}
\label{tab:outcomes2}
\end{table}





\setlength{\textfloatsep}{0.2cm}
\newcommand\mycommfont[1]{\footnotesize\ttfamily\textcolor{red}{#1}}
\SetCommentSty{mycommfont}
\begin{algorithm}[!t]
    \SetKwFunction{proc}{ConstructPY}
    \SetKwFunction{sub}{FilterNearest}
    \SetKwProg{myproc}{Procedure}{}{end}
    \SetKwProg{mysub}{Subroutine}{}{end}
    \SetKwInOut{Input}{Input}\SetKwInOut{Output}{Output}
    \DontPrintSemicolon

    \Input{Simulated observations $O$, parameter $d_{safe}$, parameter $d_{max}$, and boolean flag \textit{left} indicating whether the partition should be in the left or right lateral direction.}
    \Output{The left or right lateral partition $P$.}
    \BlankLine
    
    \myproc{\proc{$O$, $d_{safe}$, $d_{max}$, \textit{left}}}{
        $P \gets \emptyset$ \\
        \For {$o \in O$} {
            \If{$|o_x| \leq d_{safe}$} { 
                \If{$\textit{left} \wedge 0 < o_y \leq d_{max}$} {
                    \tcc{positive partition}
                    add $o$ to $P$
                }
                \ElseIf{$\neg \textit{left} \wedge -d_{max} \leq o_y < 0$}{
                    \tcc{negative partition}
                    add $o$ to $P$
                }
            }
        }
        $P \gets$ \sub{$P$, $d_{max}$} \\
        \Return{$P$}
    }
    \BlankLine
    
    \myproc{\sub{$P$, $d_{max}$}}{
        $P^\prime \gets \emptyset$ \\
        $D \gets \emptyset$ \\
        $\textit{min} \gets d_{max}$ \\
        
        \For{$o \in P$}{
            \If{$|o_y|< \textit{min}$} {
                $\textit{min} \gets |o_y|$ \\
                $D \gets o$ \\
            }     
        }
        add $D$ to $P^\prime$ \\
        \Return{$P^\prime$}
    }
\caption{Construct Lateral Partition}
\label{alg:lateralalg}
\end{algorithm}
\setlength{\textfloatsep}{0.8cm}

\subsubsection{Three-Step Sequences} 


In the event that a two-step sequence it not possible, we  check whether the robot is in a situation where it is ``boxed in'', i.e., it does not have sufficient space in either lateral direction to drive straight for a minimum distance $d_{min}$ without encountering a second disturbance. In this case, the robot can execute either the sequence $T_L \rightarrow T_L \rightarrow T_0$ or sequence $T_R \rightarrow T_R \rightarrow T_0$ to turn around and go back the way it came, otherwise a three-step plan is unnecessary and we move on to consider whether a four-step plan is possible (see Section \ref{sect:fourstep}). Options for three-step plans are restricted by our model (details provided in Section \ref{sect:trans}).  
As we restrict our attention to static environments, we can assume that returning to the default task $T_0$ after turning around is safe. 

We use the lateral dimension to reason about whether the robot can move distance $d_{min}$ in either direction. If the robot is boxed in, 
then both lateral partitions are non-empty.
Consequently, the proposition $P_L \not = \emptyset \land P_R \not = \emptyset$ forms a sufficient condition for reasoning about whether the robot has enough room to manoeuvre. As $P_L$ and $P_R$ are disjoint, we then construct a third partition: 
\begin{equation}
    P_{min} = \{D \in P_L \cup P_R\ | -d_{min} \leq D_y \leq d_{min} \} 
\end{equation}

\noindent which is a proper subset of $P_L \cup P_R$. Hence the intersection $P_{min} \cap (P_L \cup P_R)$ overlaps both partitions and includes all disturbances relevant for deciding if a three-step plan is necessary.  We then define:
\begin{equation}\label{eq:dmin}
    E = \{\langle D^L, D^R \rangle\ |\ D^L \in P_{min} \cap P_L \land D^R \in P_{min} \cap P_R \}
\end{equation}

\noindent where $D^L$ is a disturbance in partition $P_L$ and $D^R$ is a disturbance in partition $P_R$. 
If the set $E$ is non-empty, there is a disturbance at most distance $d_{min}$ in both lateral directions, hence we can reason that the robot is boxed in and a sequence of length three is generated to turn around and evade the situation.  If the set $E$ is empty, however, then a three-step plan is unnecessary.

\subsubsection{Four-Step Sequences}\label{sect:fourstep}

The final sequence length is four steps, which is required in a situation where there is a disturbance in some lateral direction but the distance is greater than $d_{min}$. Hence the robot can drive safely in some lateral direction for a while but then must counteract a second disturbance,  
such as in the sequence $T_L \rightarrow T_S \rightarrow T_L \rightarrow T_0$ which includes an intermediate straight task. As we reason about the safety of sequences in order of increasing length, by the time we have reached this point we have already reasoned that the  first subsequence $T_L \rightarrow T_S$ is safe.  Hence we only need to consider whether $T_L \rightarrow T_0$ would be safe to execute after the initial two steps.

We first simulate the lateral displacement of the robot.  We calculate the lateral offset $\Delta^L = D_y^L - d_{safe}$ and subtract this from the $y$-coordinates of observations to simulate the  expected state change from executing the subsequence $T_L \rightarrow T_0$ while respecting the safe zone. Note that as we are reasoning about the \textit{future} position of the robot \textit{relative to the current local frame} 
, lateral translation of the state-space alone is sufficient. 
It is now possible to reason about the final subsequence $T_L \rightarrow T_0 $ in the longitudinal dimension by constructing longitudinal partitions.  The procedure for this is the same whether we have simulated driving straight in either the left or right direction.

We iterate over the set of observations and construct the positive partition 
\begin{equation}\label{eq:poslong}
    P^+ = \{ o \in O\ |\ d_{safe} < o_x \leq d_{safe} + \beta d_{safe} \wedge |o_y| \leq \frac{1}{2}(L+tol)\}
\end{equation}

\noindent where $d_{safe}$ is a lower bound on the longitudinal dimension which excludes the lateral partition, $d_{safe} + \beta d_{safe}$ is an upper bound with a tunable coefficient $\beta$, and $\frac{1}{2}(L+tol)$ is a bound on the lateral dimension to respect the width of the robot plus some tolerance. In our example, we assume that a disturbance is witnessed in partition $P^+$, hence we can conclude that it is not possible to return to the default task $T_0$ and drive straight in the positive longitudinal direction. 

We then iterate over the set $O$ and construct the negative partition: 
\begin{equation}\label{eq:neglong}
    P^- = \{ o \in O\ | -(d_{safe} + \beta d_{safe}) \leq o_x < -d_{safe} \wedge |o_y| \leq \frac{1}{2}(L+tol)\}
\end{equation}

\noindent where $-(d_{safe} + \beta d_{safe})$ is a lower bound on the longitudinal dimension with the tunable coefficient $\beta$, $-d_{safe}$ is an upper bound which again excludes the lateral partition, and $\frac{1}{2}(L+tol)$ is  again a bound on the lateral dimension to respect the width of the robot. 
In our example, we assume that the negative longitudinal partition $P^-$ is empty, so we can conclude that driving straight in the negative longitudinal direction is possible
and the robot can return to the default task $T_0$.

\begin{table}[!h]
\caption{Possible outcomes for final subsequence of four-step sequences.}
\centering
\begin{tabular}{lllll}
\toprule
\textbf{Lateral Offset} & \textbf{Subsequence} & \textbf{Success} & \textbf{Failure}  \\ \midrule
$\Delta^L = D_y^L - d_{safe}$ & $T_R \rightarrow T_0$  & $P_L^+ = \emptyset $ & $P_L^+ \not = \emptyset $ \\
 & $ T_L \rightarrow T_0$  & $P_L^- = \emptyset $  & $P_L^- \not = \emptyset $     \\ 
$\Delta^R = D_y^R + d_{safe}$ & $T_R \rightarrow T_0$  & $P_R^- = \emptyset $ & $P_R^- \not = \emptyset $ \\
 & $ T_L \rightarrow T_0$  & $P_R^+ = \emptyset $  & $P_R^+ \not = \emptyset $ \\ \bottomrule 
\end{tabular}
\label{tab:outcomes4}
\end{table}

Table \ref{tab:outcomes4} summarises the possible future outcomes for final subsequences in a four-step sequence in terms of the positive and negative longitudinal partitions defined in Equations \ref{eq:poslong} and \ref{eq:neglong}, respectively.  While our example has focused on the left lateral direction, in practice both directions are assessed for completeness.  The same procedure applies to the right direction by a symmetrical argument.  Note that we refer to longitudinal partitions in the left lateral direction as $P_L^\pm$ and  in the right lateral direction as $P_R^\pm$, however this distinction is omitted where directionality is irrelevant. The procedure for constructing the longitudinal partitions is provided in Algorithm \ref{alg:longalg}.



\begin{algorithm}[!b]
    \SetKwFunction{procB}{ConstructPX}
    \SetKwProg{myproc}{Procedure}{}{end}
    \SetKwInOut{Input}{Input}\SetKwInOut{Output}{Output}
    \DontPrintSemicolon

    \Input{Simulated observations $O$, lateral offset $\Delta^{L/R}$, parameter $d_{safe}$, width of the robot plus some tolerance $L+tol$, boolean flag \textit{positive} indicating sign of partition.}
    \Output{The positive or negative longitudinal partition $P$.}
    \BlankLine
    
    \myproc{\procB{$O$, $\Delta^{L/R}$, $d_{safe}$, $L+tol$, \textit{positive}}}{
        $P \gets \emptyset$ \\
        \For {$o \in O$} {
            $o_y \gets o_y - \Delta_y$ \tcp*{simulate lateral displacement} 
            \If{$|o_y| \leq \frac{1}{2}(L+tol)$} { 
                \If{$\textit{positive} \wedge  d_{safe} < o_x \leq d_{safe} + \beta d_{safe}$} {
                    \tcc{positive partition}
                    add $o$ to $P$
                }
                \ElseIf{$\neg \textit{positive} \wedge -(d_{safe} + \beta d_{safe}) \leq o_x < -d_{safe}$}{
                    \tcc{negative partition}
                    add $o$ to $P$
                }
            }
        }
        \Return{$P$}
    }
\caption{Construct Longitudinal Partition}
\label{alg:longalg}
\end{algorithm}
\setlength{\textfloatsep}{0.8cm}

\subsection{Disturbance-Focused Transition System}\label{sect:trans} 

The abstraction approach discussed in Section \ref{sect:sim} motivates the construction of an egocentric transition system representing possible task sequences.  This can be viewed as a mental model where the transition between states encode ``core'' knowledge \cite{Spelke2007CoreKnowledge} about the temporal evolution of  closed-loop tasks and causal relationships between them.  The states themselves encode possible future disturbances which may be encountered as a consequence of task execution. 
Our model comprises $15$ states in total. However, to simplify exposition of the model and promote comprehension, we will build it up gradually before providing a formal definition. 

It is common practice to subtract a distance from an obstacle to ensure a minimum distance between the obstacle and robot is maintained for safety.  Our approach is similar through the definition of a safe zone around the robot using the distance $d_{safe}$ as indicated in Figure \ref{fig:onesuccess}B.  This is determined by the distance between the centre of mass (which in this case is aligned with the origin of the point cloud) and the furthest physical point on the robot.  To provide some room for error in the actual execution, we over-approximate this distance to arrive at the distance $d_{safe}$.  

As indicated in Figure \ref{fig:onesuccess}B, the distance $d_{safe}$ is also a lower bound on the partition $P_{shield}$ which is responsible for triggering the next task in a plan when a disturbance is witnessed during execution. As this is necessary for the execution of generated plans, the location of states in our model is always offset distance $d_{safe}$ from the future position of the origin of the robot.  While during execution of the plan this offset is applied to the $x$-coordinate, as the shield partition is fixed relative to the local frame of the robot, during planning we are reasoning about the \textit{future} position of the robot \textit{relative to the current local frame}.  Consequently, the safe zone offset $d_{safe}$ is applied to the $y$-coordinate in our model when the future orientation of the robot is orthogonal and the $x$-coordinate when parallel.  This applies to states within the lateral partitions $P_L$  and $P_R$ defined in Section \ref{sect:twostep}. 

\subsubsection{States Within Lateral Partitions}

Figures \ref{fig:tstwo}A and \ref{fig:tstwo}B show the positions of states $s_0$, $s_1$, $s_2$, $s_3$ and $s_4$ relative to the lateral partitions, required for planning two-step sequences.  The initial state $s_0$ is located directly in front of the robot with the $x$-coordinate offset distance $d_{safe}$.  As shown in the state diagram in Figure \ref{fig:statediagram}, when a disturbance is witnessed in partition $P_{look}$ a plan is generated using the model and the task switches from the default task $T_0$ to the straight task $T_S$.  As explained in Section \ref{sect:twostep}, the first step in planning is to subtract a longitudinal offset $\Delta_x = D^+ - d_{safe}$ from the $x$-coordinates of observations to simulate witnessing a disturbance in $P_{shield}$ which is a failure condition for the straight task $T_S$, triggering plan execution.  Hence the initial state $s_0$ represents a lower bound on the partition $P_{shield}$ relative to the near future location of the disturbance $D^+$.

\begin{figure}[!t]
    \centering
    \includegraphics[width=0.9\linewidth]{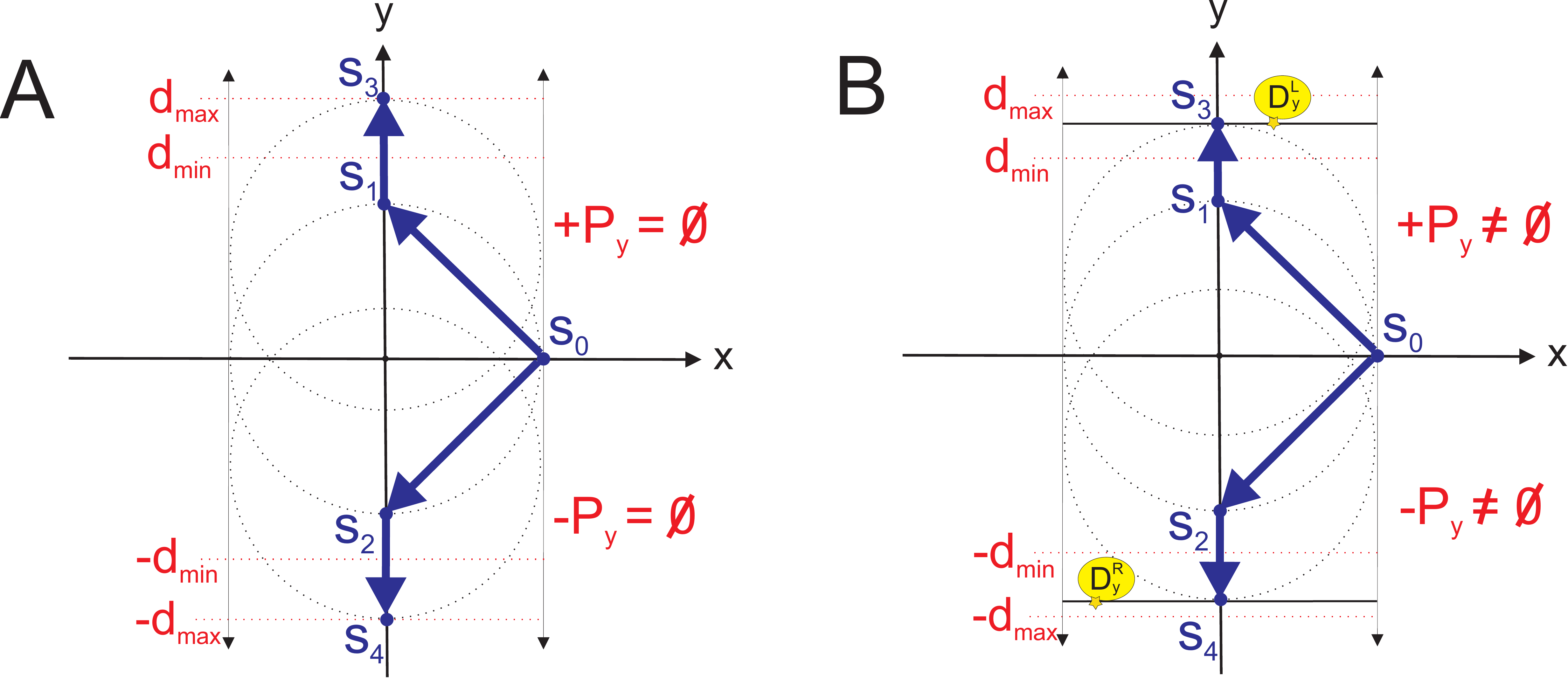}
    \Description{Fig.~\ref{fig:tstwo}. Fully described in the text.}
    \caption{States for two-step sequences. A: lateral partitions with no distal disturbances detected. B: lateral partitions with distal disturbances detected. }
    \label{fig:tstwo}
\end{figure}

States $s_1$ and $s_2$ represent the future location of the lower bound of the partition $P_{shield}$ relative to the \textit{current} local frame after executing an avoid task $T_{L/R}$.  During the actual execution of an avoid task, the robot would rotate left or right $\frac{1}{2} \pi$ radians (with some error) so the state-space would rotate in the opposite direction relative to the local frame of the robot. As the partition $P_{shield}$ is fixed relative to the local frame of the robot, the lower bound of the shield partition would be offset distance $d_{safe}$ on the $x$-coordinate.  However, as we are reasoning about the \textit{future} position of the robot relative to the \textit{current} local frame, the offset $d_{safe}$ is applied orthogonally to the $y$-coordinate.  

States $s_3$ and $s_4$ represent the future lower bound of the shield partition after driving straight in the lateral direction starting from states $s_1$ and $s_2$, respectively, again relative to the \textit{current} local frame of the robot.  However, their valuation varies depending on the location of lateral disturbances. If there are no lateral disturbances, then the future absolute lower bound on the shield partition is distance $d_{max}$, again applied to the $y$-coordinate as the future orientation of the robot for these states is orthogonal to the current frame of reference (see Figure \ref{fig:tstwo}A for an illustration).  Note that we require $d_{safe} < d_{min} < d_{max}$ when determining the value of parameter $d_{max}$ to support reasoning about multi-step plans, as described in Section \ref{sect:sim}.  While $d_{safe}$ is determined by the dimensions of the robot, there is some room for experimental tuning of parameters $d_{min}$ and $d_{max}$. 

As discussed in Section \ref{sect:formal}, the lateral boundaries defined by $d_{max}$ represent the edge of perception of the immediate environment for the robot, so it is unknown whether there are any disturbances beyond this threshold in the lateral direction. We therefore assume the robot can return to the default task $T_0$ of driving straight for an indefinite period of time.  As this is potentially infinite, by convention we stipulate that the dimension parallel to the future orientation of the robot is infinite to represent uncertainty regarding the future execution of the default task $T_0$.  During the actual execution, however, if a disturbance is eventually witnessed in $P_{look}$, this assumption is contradicted, as it is  then perceived that the future path of the robot is finite, and the planning process starts once again, though to an external observer the robot is now in a different configuration. Hence while the states $s_3$ and $s_4$ are finite, we artificially manipulate this to represent the robot's partial perception of the environment with respect to the future execution of the default task $T_0$.

However, it could be that there \textit{are} lateral disturbances, in which case this is unnecessary as the robot cannot return to the default task $T_0$ in the lateral direction. Figure \ref{fig:tstwo}B shows an illustration of this scenario with respect to the lateral partitions, which corresponds to planning an intermediate straight task $T_S$ for four-step sequences, as described in Section \ref{sect:fourstep}. The valuation of state $s_3$ or $s_4$ in this case is determined by the location of a lateral disturbance such that the future lower bound of $P_{shield}$ is equal to the $y$-coordinate of the disturbance, again relative to the \textit{current} local frame.  This represents the triggering of a subsequent avoid task $T_{L/R}$ during the actual execution.  

We summarise the possible valuations of states within the lateral partitions in Table \ref{tab:lateralstates}. Note that states $s_3$ and $s_4$ have two possible valuations for the reasons previously discussed.  In our model, for the most part odd numbered states are to the left of the robot and even states are to the right. 

\begin{table}[!h]
\caption{Possible valuations of states within lateral partitions.}
\centering
\begin{tabular}{lllll}
\toprule
$\mathbf{P}$ & $\mathbf{s}$ & $\mathbf{x}$ & $\mathbf{y}$ & $\mathbf{\theta}$  \\ \toprule
$-$ $ $ $ $ $ $ & $s_0$ $ $ $ $ $ $ $ $ & $d_{safe}$  & $0$ & $0$ \\
\midrule
$P_L$ & $s_1$ & $0$  & $d_{safe}$ & $ \pi / 2$ \\
$ $ & $s_3$ & $0$  & $\infty$ & $\pi / 2$ \\
$ $ & $s_3$ & $0$  & $D^L_y$ & $\pi / 2$ \\
\midrule
$P_R$ & $s_2$ & $0$  & $-d_{safe}$ & $-(\pi / 2)$ \\
$ $ & $s_4$ & $0$  & $-\infty$ & $-(\pi / 2)$ \\
$ $ & $s_4$ & $0$  & $D^R_y$ & $-(\pi / 2)$ \\
\bottomrule 
\end{tabular}
\label{tab:lateralstates}
\end{table}

\subsubsection{States Within Longitudinal Partitions}

If both lateral partitions have disturbances with a $y$-coordinate which is less than $d_{min}$, then the robot is ``boxed in'' and a three-step plan is necessary.  This situation is represented in Figure \ref{fig:composite}A, which also shows two states required for three-step plans, $s_{13}$ and $s_{14}$.  State $s_{13}$ represents the future location of the lower bound of the partition $P_{shield}$ after executing two consecutive avoid tasks in the same direction from the initial state---first to $s_1$ or $s_2$ then to state $s_{13}$. During the actual execution, the robot would rotate left or right $\pi$ radians (with some error) and consequently the state-space would rotate in the opposite direction relative to the local frame of the robot.  As the partition $P_{shield}$ is fixed relative to the local frame of the robot, the offset to the location of the lower bound of the shield partition would therefore be $d_{safe}$, applied to the $x$-coordinate. However, when reasoning about the \textit{future} orientation of the robot relative to the \textit{current} local frame reference, the offset from the origin is $-d_{safe}$, again applied to the $x$-coordinate. State $s_{14}$ represents the lower bound of the shield partition for a subsequent default task $T_0$.   As this represents driving straight in the opposite direction, it is assumed that there are no disturbances. Hence the $x$-coordinate of $s_{14}$ is assumed infinite until a disturbance contradicts that assumption.

\begin{figure}[!t]
    \centering
    \includegraphics[width=0.9\linewidth]{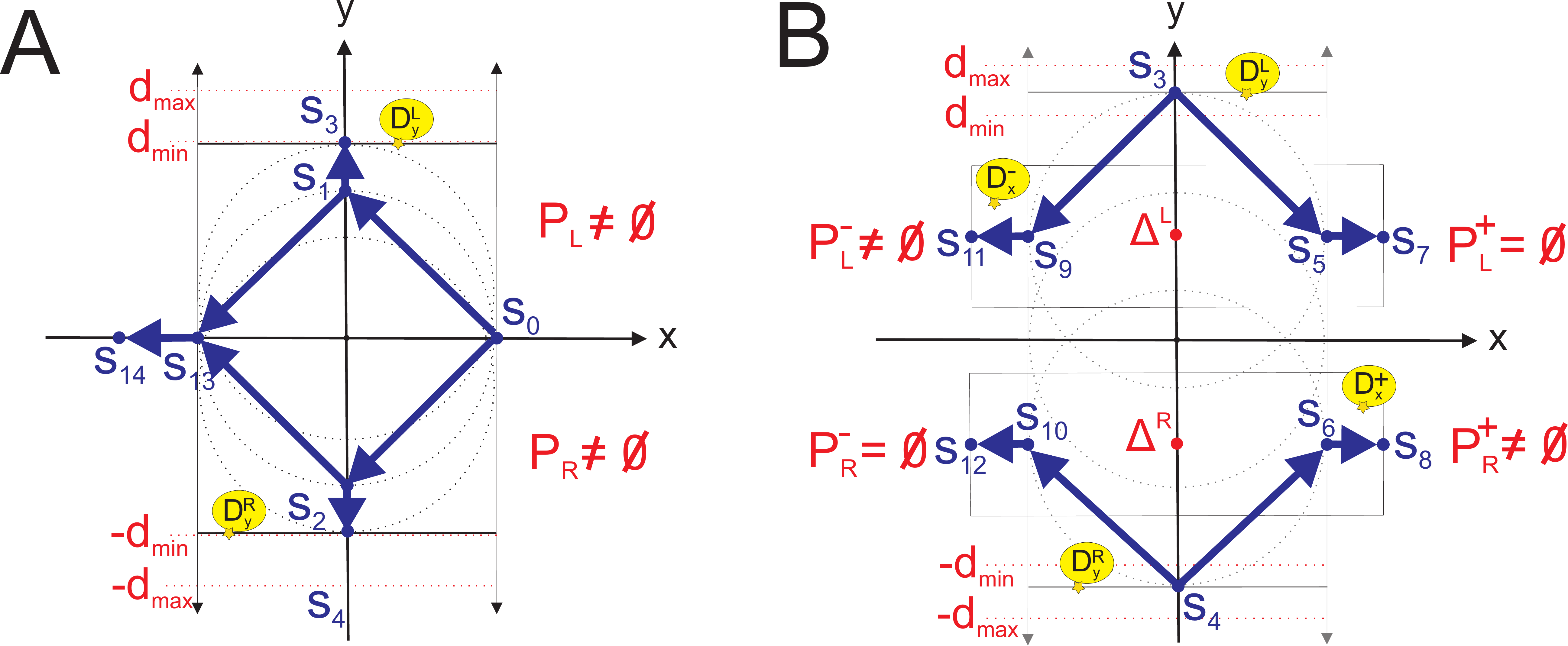}
    \Description{Fig.~\ref{fig:composite}. Fully described in the text.}
    \caption{States for three- and four-step sequences. A: states for three-step sequence. B: states for the final subsequence of a four-step sequence.}
    \label{fig:composite}
\end{figure}

The remaining states are required for four-step sequences.  States $s_5$, $s_6$, $s_9$ and $s_{10}$ represent the lower bound of the shield partition after first driving straight in some lateral direction then subsequently executing an avoid task $T_{L/R}$.  During the actual execution of a straight task $T_S$ in the lateral direction, the robot would witness a disturbance in the partition $P_{shield}$ (fixed relative to the local frame of the robot), which would then trigger the next task in a plan, an avoid task $T_{L/R}$.  The robot would then rotate $\frac{1}{2} \pi$ radians (with some error) left or right and the state-space would rotate in the opposite direction from the robot perspective.  From the perspective of the \textit{current} local frame of the robot, however, the future position of the robot origin will be offset distance $d_{safe}$ from the $y$-coordinate of the either $s_3$ or $s_4$ which is equal to the $y$-coordinate of a lateral disturbance. We denote this offset $\Delta_L = D^L_y - d_{safe}$ for state $s_3$ and $\Delta_R = D^R_y + d_{safe}$ for state $s_4$.  

The final states are $s_7$, $s_8$, $s_{11}$ and $s_{12}$ which represent the future lower bound of the shield partition and the possibility of returning to the default task $T_0$, starting from one of the states discussed in the previous paragraph.   As with states $s_3$ and $s_4$, the valuation of these states varies depending on whether there are any disturbances in the associated longitudinal partition (see Figure \ref{fig:composite}B).  However, the decision in this case is binary---we only need to know if we can return to the default task $T_0$ or not. If a disturbance is in the associated partition, then we say that the absolute value of the $x$-coordinate of the state is finite and equal to $\beta d_{safe}$, where $\beta$ is a coefficient for tuning the longitudinal dimension of the partition.  Otherwise, there is no disturbance and we consequently assume that the $x$-coordinate is infinite to represent the unknown future state of the default task $T_0$. The possible valuations of states within longitudinal partitions are summarised in Table \ref{tab:longstates}. 

\begin{table}[!h]
\caption{Possible valuations of states within longitudinal partitions.}
\centering
\begin{tabular}{lllll}
\toprule
$\mathbf{P}$ & $\mathbf{s}$ & $\mathbf{x}$ & $\mathbf{y}$ & $\mathbf{\theta}$  \\ \toprule
$-$ $ $ $ $ $ $ & $s_{13}$ $ $ $ $ $ $ $ $ & $-d_{safe}$ & $0$ & $\pm \pi$ \\
$-$ $ $ $ $ $ $ & $s_{14}$ $ $ $ $ $ $ $ $ & $-\infty$ & $0$ & $\pm \pi$ \\
\midrule
$P^+_R$ & $s_6$ & $d_{safe}$  & $\Delta_R$ & $0$ \\
$ $ & $s_8$ & $\beta d_{safe}$  & $\Delta_R$ & $0$ \\
$ $ & $s_8$ & $\infty$  & $\Delta_R$ & $0$ \\
\midrule
$P^+_L$ & $s_5$ & $d_{safe}$  & $\Delta_L$ & $0$ \\
$ $ & $s_7$ & $\beta d_{safe}$  & $\Delta_L$ & $0$ \\
$ $ & $s_7$ & $\infty$  & $\Delta_L$ & $0$ \\
\midrule
$P^-_R$ & $s_{10}$ & $-d_{safe}$  & $\Delta_R$ & $ \pi$ \\
$ $ & $s_{12}$ & $-\beta d_{safe}$  & $\Delta_R$ & $\pi$ \\
$ $ & $s_{12}$ & $-\infty$  & $\Delta_R$ & $\pi$ \\
\midrule
$P^-_L$ & $s_9$ & $-d_{safe}$  & $\Delta_L$ & $ \pi$ \\
$ $ & $s_{11}$ & $-\beta d_{safe}$  & $\Delta_L$ & $\pi$ \\
$ $ & $s_{11}$ & $-\infty$  & $\Delta_L$ & $\pi$ \\
\bottomrule 
\end{tabular}
\label{tab:longstates}
\end{table}

\subsubsection{Formal Definition of Our Model}
Now that we have described the set of states in our model, we formally define our finite transition system, illustrated in Figure \ref{fig:ts}.

\begin{definition}\label{dfts}[Disturbance-Focused Transition System]
   A disturbance-focused transition system $DTS$ is a tuple $(S, \mathcal{T}, \rightarrow, I, \Phi, L)$ where

\begin{itemize}
    \item $S = \{s_0, s_1,...,s_{14}\}$ is a set of states for which the orientation of the robot is fixed, the longitudinal and lateral dimensions are fixed for states $\{s_0, s_1, s_2, s_{13}, s_{14}\}$, and for the remaining states the longitudinal or lateral dimensions vary in relation to disturbances;
    \item $\mathcal{T} = \{T_0, T_S, T_L, T_R\}$ is the set of tasks defined in Section \ref{sect:tasks};
    \item $\rightarrow \subseteq S \times \mathcal{T} \times S$ is a transition relation where $s \rightarrow s^\prime$ is admissible if and only if a task can evolve the model from state $s$ to $s^\prime$;
    \item $I = \{s_0 = [d_{safe}, 0, 0]\}$ is the initial state; 
    \item $\Phi$ is a set of state formulae;
    \item $L : S \rightarrow 2^{\Phi}$ is a labelling function. 
\end{itemize}
\end{definition}

\noindent The states in our model are labelled with formulae from $\Phi$ at runtime using the abstraction approach described in Section \ref{sect:sim}, based on the location of disturbances. The state formulae are defined inductively from the set of atomic propositions $AP$.  We provide definitions of two state formulae for our model and examples of label assignment in the next section. 

\begin{figure}[!h]
    \centering
    \includegraphics[width=0.4\linewidth]{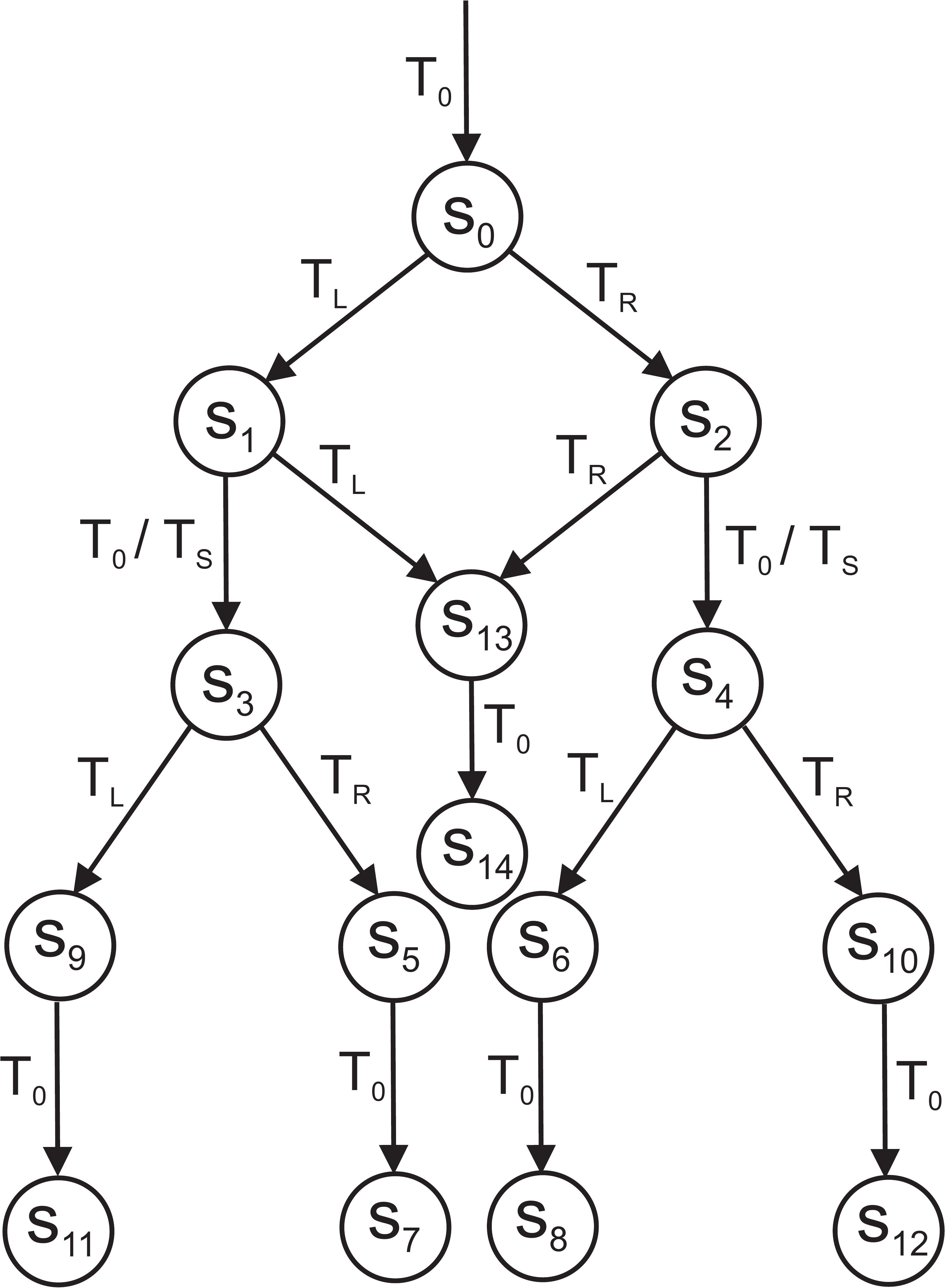}
    \Description{Fig.~\ref{fig:ts}. Disturbance-focused transition system, a tree-like structure representing all possible states and transitions.}
    \caption{Disturbance-focused transition system.}
    \label{fig:ts}
\end{figure}

\subsection{Planning Using Model Checking}\label{sect:planning}

We use the \textit{disturbance-focused} transition system described in the previous section to generate plans to negotiate multiple disturbances using distal sensor information. To do this, we define an LTL property, construct a non-deterministic finite automaton (NFA) and form the product transition system and NFA for a labelling  of state properties based on sensor data. 

We implement the underlying directed graph for the finite transition system shown in Figure \ref{fig:ts} and compute the product at runtime when a new planning sequence is initiated. Hence planning using model checking can be reduced to a reachability analysis on the product, which amounts to simply deciding which states in the graph are terminal. As indicated in Section \ref{sect:sim}, we reason about the transition system in order of increasing sequence length.  This ensures that reasoning about safety accumulates as plans increase in length.  

\subsubsection{LTL Plan Specification}\label{sect:ltl}

To specify plans for our approach, we define a regular property in LTL which is stated in terms of a set of two state properties, $horizon$ and $safe$. Before introducing the LTL property for our approach, we define these state properties.  


\begin{definition}\label{defn:horizon} [State Property horizon] Property $horizon$ is a state formula in LTL:

\begin{equation}\label{eq:horizon}
    horizon = (|x| = \infty \land \neg |y| = \infty) \vee  (\neg |x| = \infty \land |y| = \infty)
\end{equation}

\noindent where $(|x| = \infty, |y| = \infty) \in AP$ and $x$ and $y$ are the coordinates of a state in $C$-space. Property $horizon$ is true at $s$ if the valuation of either the longitudinal or lateral dimension (but not both) is infinite.


\end{definition}

In Definition \ref{defn:horizon} an infinite dimension in the valuation of a state means that the robot can drive in a straight line for a (potentially) infinite period of time in the longitudinal or lateral direction, so it can safely return to the default task $T_0$. If the future path of the robot is deemed potentially infinite, then there is no known disturbance ahead, so by definition of property $horizon$ the state is safe. 



\begin{definition}\label{defn:safe} [State Property safe] Property $safe$ is a state formula in LTL:
\begin{subequations}
\begin{align}
  sa\mathit{f}e =\ &horizon\label{defn:safe:cond1} \\ 
  & \vee ((|x| = d_{safe} \land \neg |y| = d_{safe}) \vee (\neg |x| = d_{safe} \land |y| = d_{safe})) \label{defn:safe:cond2}\\
                          & \vee (x  = 0 \land (d_{min} < y < D_y^{L} \vee D_y^{R} < y < d_{min})) \label{defn:safe:cond3}
\end{align}
\end{subequations}

\noindent where $horizon \in \Phi$, $(|x| = d_{safe}, |y| = d_{safe}, x = 0, d_{min} < y < D_y^{L}, D_y^{R} < y < d_{min}) \in AP$, $x$ and $y$ are state coordinates in $C$-space, $d_{safe}$ is a parameter defining the safe zone, $d_{min}$ is a parameter defining the minimum distance the robot must be able to drive in the lateral direction, and $D^L_y$ and $D^R_y$ define the $y$-coordinates of possible disturbances in the left and right direction, respectively. Hence property $safe$ is true at state $s$ if one of the following conditions hold: (\ref{defn:safe:cond1}) state property $horizon$ is true at $s$; (\ref{defn:safe:cond2}) the absolute value of one finite dimension at $s$ is $d_{safe}$; or (\ref{defn:safe:cond3}) $s$ has a finite lateral dimension  such that the absolute distance of a lateral disturbance is  greater than $d_{min}$ and $x$ is zero.
\end{definition}

In Definition \ref{defn:safe} sub-formula (\ref{defn:safe:cond1}) guarantees that a state is safe when the future path seems infinite.  In our approach, we always start in the initial state $s_0$ and assess each state for a possible path in order of preference, hence for any final horizon state we always 
know that the path leading to state $s$ is finite. Sub-formula (\ref{defn:safe:cond2}) ensures rotations are safe due to construction of the safe zone and outer shield for witnessing proximal disturbances.  The value of one dimension is equal to $d_{safe}$ for relevant states around the origin and this uniquely identifies them.  
The final sub-formula (\ref{defn:safe:cond3}) ensures that the robot has space to move in the lateral direction while respecting the safe zone. 




Based on these state properties, we define the following path property for planning:

\begin{equation}\label{eq:property}
    \varphi = \neg (safe\ \mathbf{U}\ (safe \land horizon))
\end{equation}

\noindent which states that it is not true that $safe$ holds until $safe \land horizon$.  All LTL properties have an implicit universal quantifier, i.e. an assumed  \textit{for all paths} operator. So $\varphi$ is read as: $\neg (safe\ \mathbf{U}\ (safe \land horizon))$ holds for every path.  Note that a finite path satisfies $safe\ \mathbf{U}\ (safe \land horizon)$ if both $safe$ and $horizon$ are true at some state $s$ in the path, and $safe$ is true for all states in the path leading to $s$. As our interest is in \textit{solution paths} (see Section \ref{sect:formal}), the property $\varphi$ negates the desired outcome, so that what would normally be the set of counter-examples for a run of the system, in our case becomes a set of safe \textit{future discrete task outcomes}.  Consequently, the set of solution paths is the set of paths which satisfy $safe\ \mathbf{U}\ (safe \land horizon)$.  Our property and modelling approach is intentionally simple, primarily to ensure fast computation and minimal resource usage on a low-powered device.

\begin{figure}[!h]
    \centering
    \includegraphics[width=0.4\linewidth]{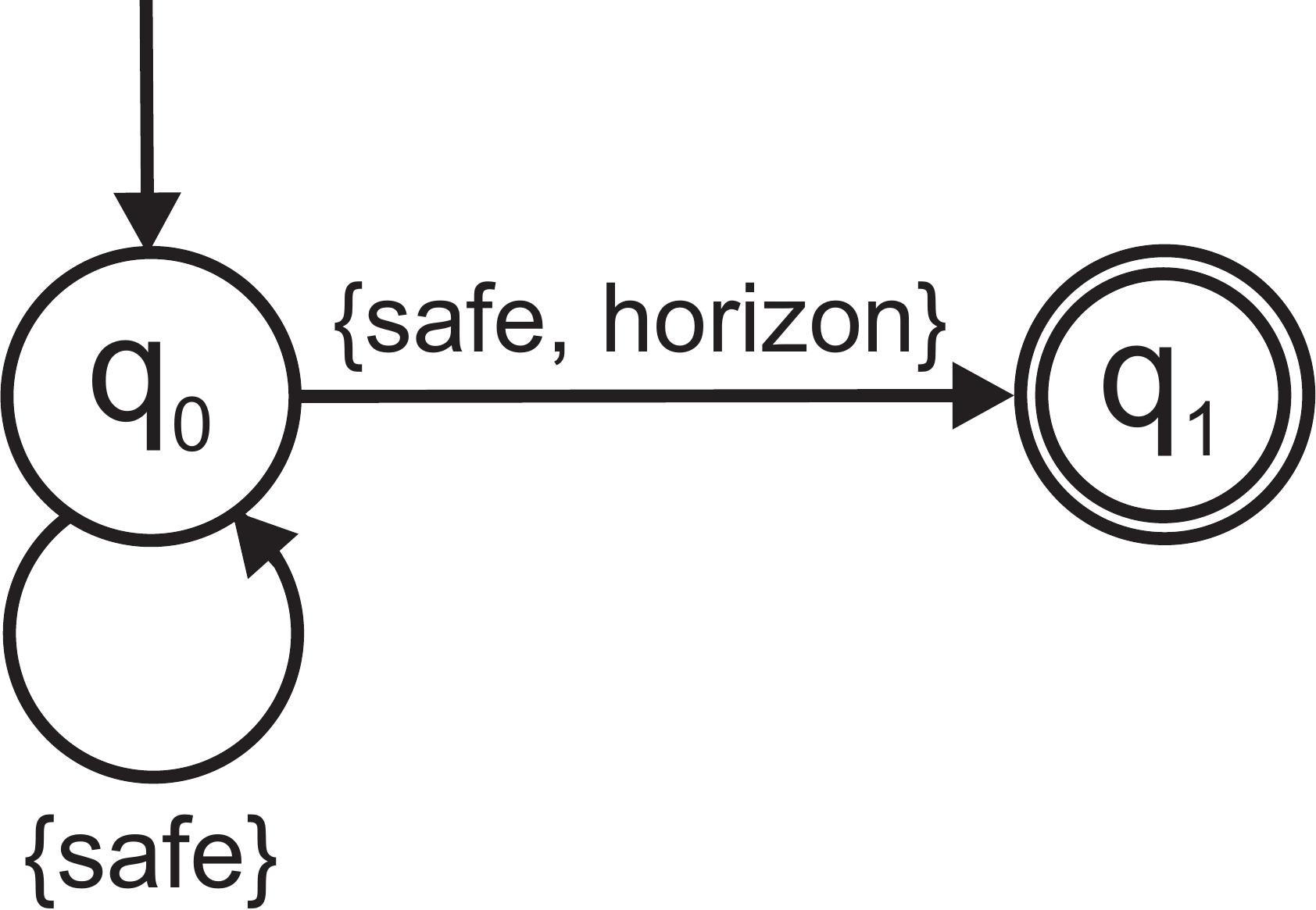}
    \Description{Fig.~\ref{fig:automaton}. Non-deterministic finite automaton with two states, q 0 and q 1.  The initial states is q 0 with a self loop for the singleton set containing the atomic proposition safe.  The transition to the accepting state q 1 is activated for the set containing both atomic propositions safe and horizon.}
    \caption{NFA for the property $safe\ \mathbf{U}\ (safe \wedge horizon) $.}
    \label{fig:automaton}
\end{figure}

\subsubsection{Checking the Product Transition System}\label{sect:product}

The set of counter-examples for $\varphi$ constitute a language of finite words which can be recognised by an NFA. We therefore construct the NFA $\mathcal{A}_{\neg\varphi} = (Q, \Sigma, \delta, Q_{0}, F)$ where $Q = \{q_0, q_1\}$ is the set of states, $\Sigma = 2^{\Phi}$ is a finite alphabet on the set of state formulae $\Phi$, $\delta : Q \times \Sigma \rightarrow 2^{Q}$ is a transition relation, $Q_{0} = \{q_0\}$ is the initial state, and $F = \{q_1\}$ is the accepting state.  Figure \ref{fig:automaton} shows an illustration of the NFA $\mathcal{A}_{\neg\varphi}$.

When a new disturbance is encountered by the robot, we use the abstraction approach explained in Section \ref{sect:sim} to assign a valuation to the corresponding state from the set $\Phi$.  We then calculate the automaton formed as the product of $DTS$ and $A_{\neg\varphi}$ and generate a (finite) accepting path non-deterministically using forward depth-first search (f-DFS).  This starts in the root node of the underlying graph (i.e., the initial state $s_0$) and backtracks once a leaf node is reached, which is appropriate  as the goal state not determined a priori. From the path we then extract the transitions to form a plan for the scenario. Figure \ref{fig:product}A shows an instance of our transition system with valuation of the state formulae in $\Phi$ and Figure \ref{fig:product}B shows the product automaton. Note that accepting (solution) paths are those for which the NFA part of the state in Figure \ref{fig:product}B is the accepting state $q_{1}$.

\begin{figure}[!h]
    \centering
    \includegraphics[width=\linewidth]{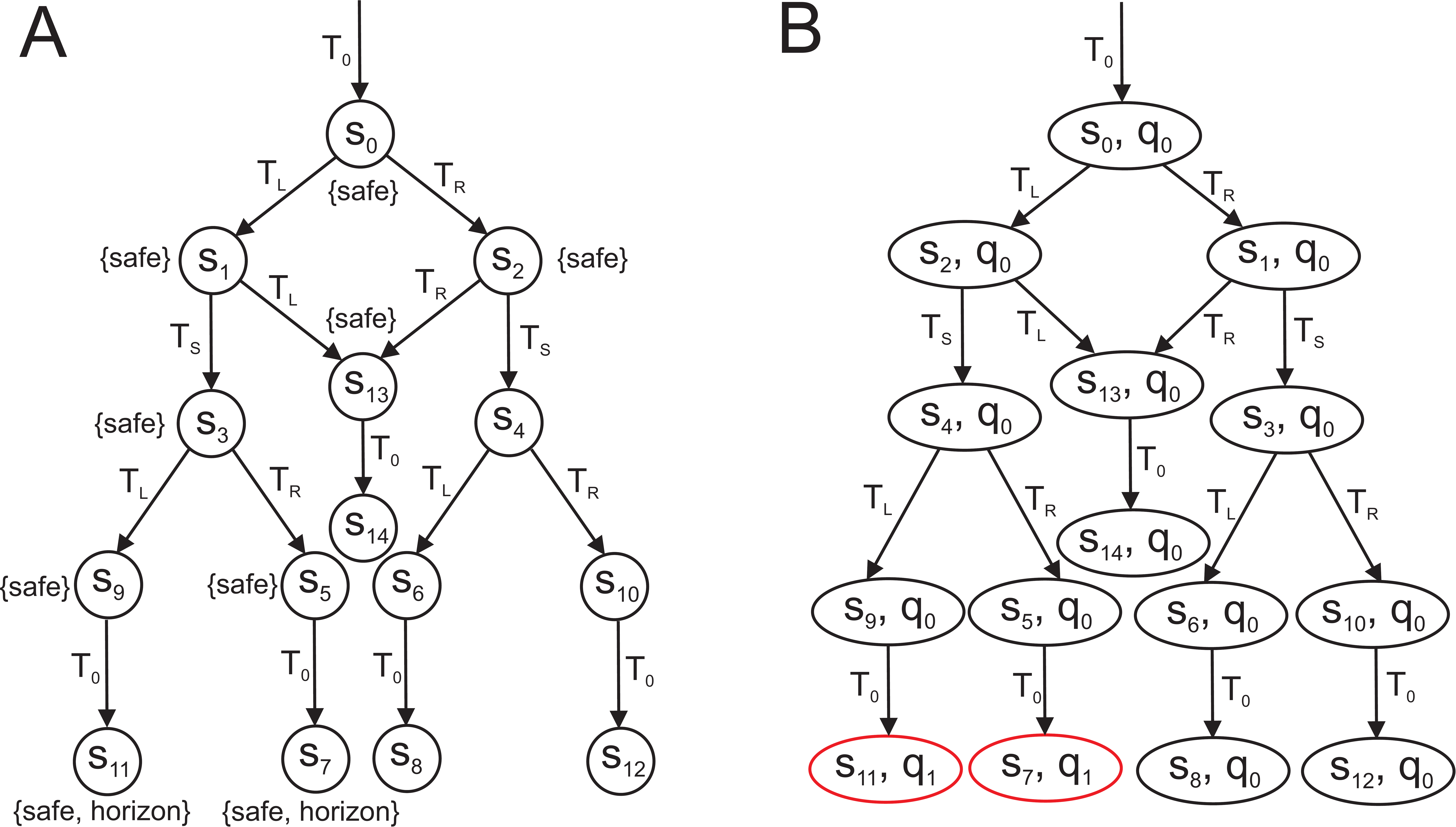}
    \Description{Fig.~\ref{fig:product}. A: shows an an instance of the disturbance-focused transition system with a valuation of atomic propositions.  The propositions safe and horizon hold in states 11 and 7 and atomic proposition safe holds for all states leading to them.  B: shows the product transition system and NFA for this valuation, with states 11 and 7 highlighted in red as accepting states.}
    \caption{A: instance of our transition system showing a valuation of the elements of $AP$. B:  product transition system and NFA resulting from the valuation in A.}
    \label{fig:product}
\end{figure}

\subsubsection{Runtime Generation of Plans}\label{sect:generation}
Algorithm \ref{alg:planningalg} provides details on the runtime procedure for computing our abstraction and generating plans.  As mentioned above, we assess possible sequences in order of length to determine a valuation of our transition system.  Note that the subroutine \textit{generatePlan()} finds a solution path by f-DFS on the product transition system and NFA.  Starting from the initial state $s_0$, it then extracts the transitions between states to form a plan consisting of a sequence of tasks.  By default f-DFS does not return the accepting state of the product, so the final transition is not returned.  However, this is sufficient, as the final task is always the default task $T_0$.


\begin{algorithm}[!t]
    \SetKwFunction{plan}{GenerateNewPlan}
    \SetKwProg{myproc}{Procedure}{}{}
    \SetAlgoLined
    \SetKwInOut{Input}{Input}\SetKwInOut{Output}{Output}
    \DontPrintSemicolon

    \Input{Set of observations $O$ and disturbance $D^+$.}
    \Output{Sequence of discrete control tasks \textit{plan}.}
    \BlankLine

    \myproc{\plan{$O$, $D^+$}}{
        \SetAlgoVlined
        $\textit{term} \gets \emptyset$ \tcp*{set of terminal states for product}
        $\Delta^+ \gets 0$ \tcp*{simulate proximal detection of $D$}
        \If{$\Delta^+ > d_{safe} $} {
            $\Delta^+ = D_x^+ - d_{safe}$ \\
        }

        $O^\prime \gets \emptyset$ \\
        \For{$o \in O$} {
            $o_x \gets o_x - \Delta^+$ \\
            add $o$ to $O^\prime$ \\
        }
        \BlankLine

        $P_L \gets$ \proc{$O^\prime$, $d_{safe}$, $d_{max}$, \textbf{true}} \tcp*{assess two steps}
        $P_R \gets$ \proc{$O^\prime$, $d_{safe}$, $d_{max}$, \textbf{false}} \\
        \BlankLine

        \If{$P_L = \emptyset$}{
            add $s_3$ to \textit{term} \\
        }
        \If{$P_R = \emptyset$}{
            add $s_4$ to \textit{term}  
        }     
        \If{$P_L = \emptyset \vee P_R = \emptyset$}{
            $\textit{plan} \gets \textit{generatePlan}(\textit{term})$ \\
            \Return{$\textit{plan}$} \tcp*{assess three steps}
        } 
        \BlankLine

        \If{$D_y^L \in P_L < d_{min} \wedge D_y^R \in P_R > -d_{min}$}{
            add $s_{14}$ to \textit{term} \\
            $\textit{plan} \gets \textit{generatePlan}(\textit{term})$ \\
            \Return{$\textit{plan}$} \\
        }
        \BlankLine

        $\Delta^L \gets D_y^L - d_{safe}$ \tcp*{assess four steps}
        $\Delta^R \gets D_y^R + d_{safe}$ \\
   
        $P_L^+ \gets$ \procB{$O^\prime$, $\Delta^L$, $d_{safe}$, $L+tol$, \textbf{true}} \\
        $P_L^- \gets$ \procB{$O^\prime$, $\Delta^L$, $d_{safe}$, $L+tol$, \textbf{false}} \\
        $P_R^+ \gets$ \procB{$O^\prime$, $\Delta^R$, $d_{safe}$, $L+tol$, \textbf{true}} \\
        $P_R^- \gets$ \procB{$O^\prime$, $\Delta^R$, $d_{safe}$, $L+tol$, \textbf{false}} \\
        \BlankLine
       
        \If{$P1_x^+ = \emptyset$}{
            add $s_7$ to \textit{term}
        }
        \If{$P1_x^- = \emptyset$}{
            add $s_{11}$ to \textit{term}
        }
        \If{$P2_x^+ = \emptyset$}{
            add $s_8$ to \textit{term}
        }
    }
\caption{Generate Plan Using Model Checking}
\label{alg:planningalg}
\end{algorithm}
\setlength{\textfloatsep}{0.8cm}

\setcounter{algocf}{2}
\begin{algorithm}[!h]
    \setcounter{AlgoLine}{5}
    \SetKwFunction{plan}{}
    \SetKwProg{myproc}{}{}{end}
    \DontPrintSemicolon

    \nlnonumber
    \myproc{}{




   
       
        \If{$P_R^- = \emptyset$}{
            add $s_{12}$ to \textit{term}
        }
        $\textit{plan} \gets \textit{generatePlan}(\textit{term})$ \\
        \Return{$\textit{plan}$} \\
    }
    \caption{(continued)}
\end{algorithm}
\setlength{\textfloatsep}{0.8cm}

\section{Design of Case Study and Evaluation}\label{sect:eval}


In this section, we address the following research question: \\

\hspace*{0.5cm}%
\begin{minipage}{.8\textwidth}%
       \textbf{RQ} Does our approach to planning using model checking and closed-loop tasks reliably generate safe obstacle avoidance behaviour? 
\end{minipage}%
\vspace{0.5cm}

\noindent Our goal was to investigate whether our method generates safe trajectories for obstacle avoidance compared to a  reactive agent that can only generate a simple reflex response.  In addition, we aimed to demonstrate that multi-step planning means that our approach cannot lead to our robot becoming trapped in a cul-de-sac compared to single-step planning (see Section \ref{sect:motivation} for details).  
We also wanted to gain insight into the reliability of real-time performance with respect to a deadline of 100 milliseconds and investigate the memory usage of our method. We used two artefacts in our case study: a baseline method and our model checking approach.  Our baseline method was a simple reactive controller that can spawn a single closed-loop task at a time in response to disturbances. 

In Section \ref{sect:culdesac}, we conduct a quantitative assessment of our model checking approach compared to the baseline for a cul-de-sac scenario.  We investigate the length of generated trajectories as a proxy for time spent in the cul-de-sac and compare them statistically; we also count the number of collisions and report the latency and memory usage for both methods. In Section \ref{sect:playground}, we provide some qualitative results showing our method running in a bounded environment, and again count the collisions and report both the latency and memory usage.  We further demonstrate that our method is cross-platform in Appendix \ref{app:cross} and provide informal theoretical results in Appendix \ref{app:theory}.



\subsection{Scenario 1 - Cul-de-sac}\label{sect:culdesac}



Figure \ref{fig:figure8}A shows an overview of the experimental setup for our quantitative analysis of the cul-de-sac scenario.  In the schematic, there are three starting positions (left, centre, right) with the orientation of the robot defined relative to the bottom wall.  For the centre position, the orientation of the $x$-axis is offset $0$ degrees relative to the bottom wall, but for the left and right starting positions it is offset $45$ degrees. There were two reasons for this setup: (i) we wanted to maximise the chances of the baseline method causing the robot to get trapped, and (ii) pilot experiments revealed that the robot had a strong right veer when driving straight, hence we balanced our study.

\begin{figure}[!h]
    \centering
    \includegraphics[width=0.7\linewidth]{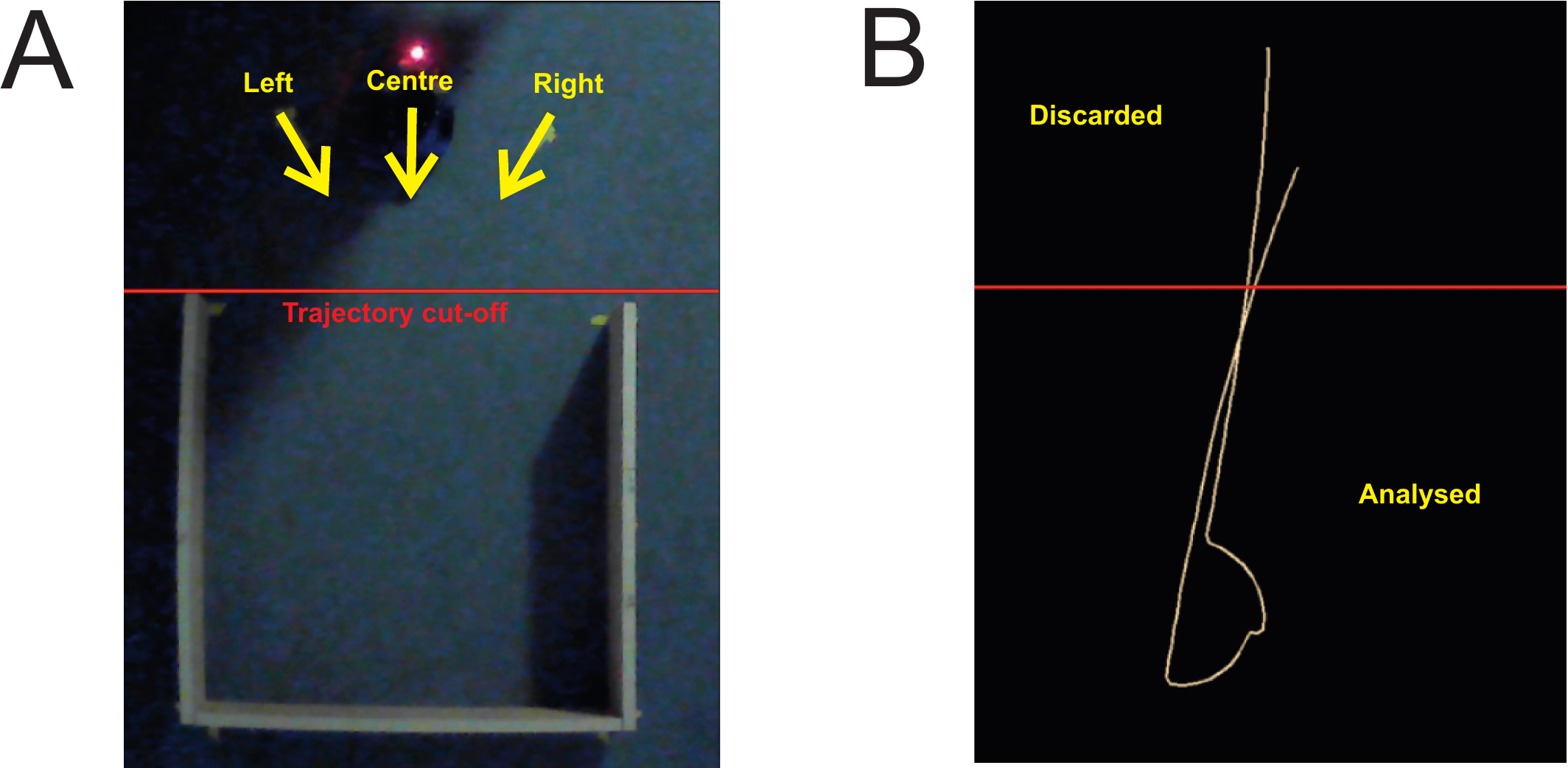}
    \Description{Fig.~\ref{fig:figure8}. A: shows and ariel photograph of the experimental setup with starting positions indicated and placement of the trajectory cut-off at the entrance of the cul-de-sac. B: shows an example trajectory indicating the portion discarded and analysed using the trajectory cut-off.}
    \caption{The experimental setup. A:  the three starting positions (left, centre, right) and placement of the trajectory cut-off. The left most marker is the left starting position, the robot occupies  the centre position, and the marker to the right of the robot is the right starting position.  B: example trajectory and truncation according to the cut-off in A. Below the cut-off is analysed and above the cut-off is discarded. }
    \label{fig:figure8}
\end{figure}




\textit{Design of Analysis}. Our study had a $2 \times 3$ between-factors design for the statistical analysis of trajectory lengths, where the length of a trajectory was taken as a proxy measure for the amount of time spent in the cul-de-sac. The method variable had two levels, ``Single'' and ``Multi'', representing the baseline method and our multi-step planning approach using model checking, respectively.  The starting position variable had three levels (``Left'', ``Centre'', and ``Right'') which are indicated in Figure \ref{fig:figure8}A. We collected 15 runs for each method and starting position, resulting in a dataset of 90 runs.   
Note that the red line in Figure \ref{fig:figure8}A represents a trajectory cut-off used to prepare trajectories for analysis, as indicated in Figure \ref{fig:figure8}B.  This was to ensure that all trajectories had the same end-point when exiting the cul-de-sac for comparability in the statistical analysis. 

\textit{Data Collection and Preparation}.  To measure the length of trajectories, we recorded each run of the robot on video and extracted the generated trajectories using the optical flow method available in OpenCV\footnote{\url{https://opencv.org/}} (see Figure \ref{fig:figure8}B for an example).  The optical flow method picks up on the brightest pixel on a frame-by-frame basis, generating a trace between frames. To ensure that only the LED on the robot was traced, we had to keep the environment relatively dark. However, the procedure for extracting trajectory lengths was semi-automatic, as other bright pixels can still be traced. The length of the trajectory was calculated in pixels. We used tape markers to indicate the correct placement of the robot for repeatability, one for each of the left, centre and right starting positions, as indicated in Figure \ref{fig:figure8}A (obscured by the arrows).  
Runs were terminated when the robot was observed to pass the trajectory cut-off threshold of the cul-de-sac by a visibly clear margin.



\textit{Results}. We first compared trajectory lengths for each method while ignoring  starting position.  
A non-parametric Mann-Whitney U test was performed. This was chosen because there was no evidence or prior knowledge that the distribution of trajectory lengths approximated a Gaussian distribution, so the mean of the samples could not be used as a measure of central tendency.  A Mann-Whitney U test can be interpreted as a comparison of the median of two samples. However, it should be noted that this is only if the shape and dispersion of the two samples are the same, otherwise the test statistic $U$ reflects a comparison of the mean ranks.  As the shape and dispersion of the distributions were different (see Figure \ref{fig:figure9} for boxplots), the null hypothesis was that the mean ranks for each method are identical, and the alternative hypothesis was that it is lower for the model checking group, $\alpha = .025$ (one-tailed).  The test revealed that the mean rank for model checking was significantly lower, $U = 251$, $n_1/n_2 = 45$, $p < .001$, with a medium effect size, $r = .37$. 

\begin{figure}[!h]
    \centering
    \includegraphics[width=0.8\linewidth]{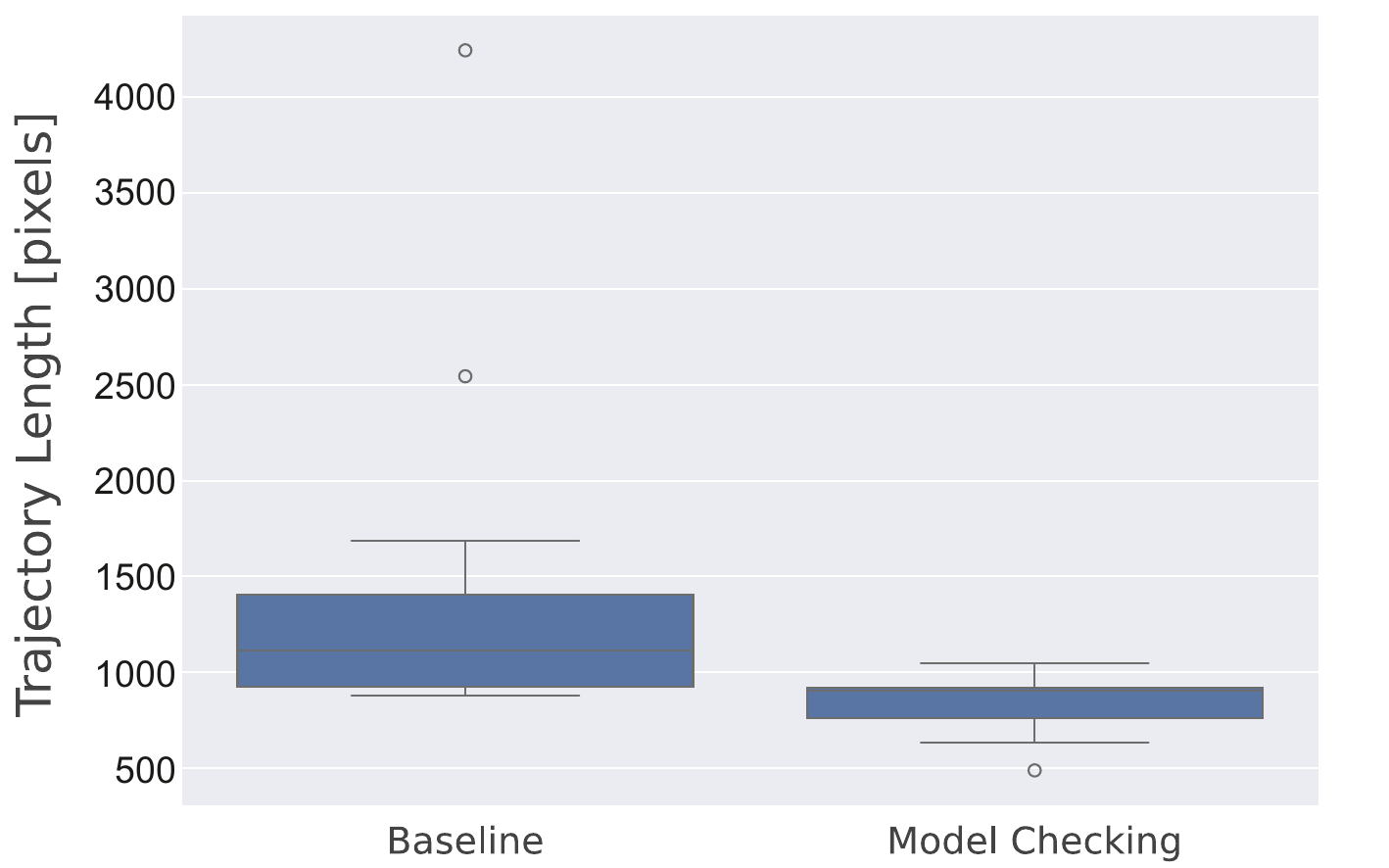}
    \Description{Figure \ref{fig:figure9}. Box plot of trajectory lengths in pixels grouped by method.  Average trajectory length is higher for the baseline method and shows more variance than for planning using model checking.  }
    \caption{Trajectory length grouped by method.}
    \label{fig:figure9}
\end{figure}

We then made a comparison between starting positions.  It was reasoned that this could have an influence on the generated trajectory lengths due to differences in starting pose, specifically the angle of entry into the cul-de-sac. 
We performed a non-parametric Kruskal-Wallis test on starting position and trajectory length for each method individually, again because there was no evidence or prior knowledge that the distribution of the groups was approximately Gaussian. 

For the baseline method, the test revealed that there was a significant effect of starting position on trajectory length $H(2) = 12.17, p = .002$.  Note that $H(2)$ is the test statistic for the Kruskal-Wallis test, in this case indicating two degrees of freedom (because in this analysis there are three groups, one for each starting position).  As the statistic $H(2)$ was significant, it indicated that there was a difference between starting position for the baseline method, hence post-hoc analysis of pairwise comparisons was required to reveal the nature of the difference.  Pairwise comparisons using Dunn's test (Bonferroni correction applied) indicated that there was a significant difference between the centre starting position and both the left ($p = .026$) and right ($p = .002$), while the left starting position was not significantly different from the right. The results suggested that the baseline method produces similar trajectory lengths when entering the cul-de-sac at an angle of $\approx 45$ degrees. See Figure \ref{fig:figure10}A for a comparison of the distributions for each starting position.  

\begin{figure}[!t]
    \centering
    \includegraphics[width=0.9\linewidth]{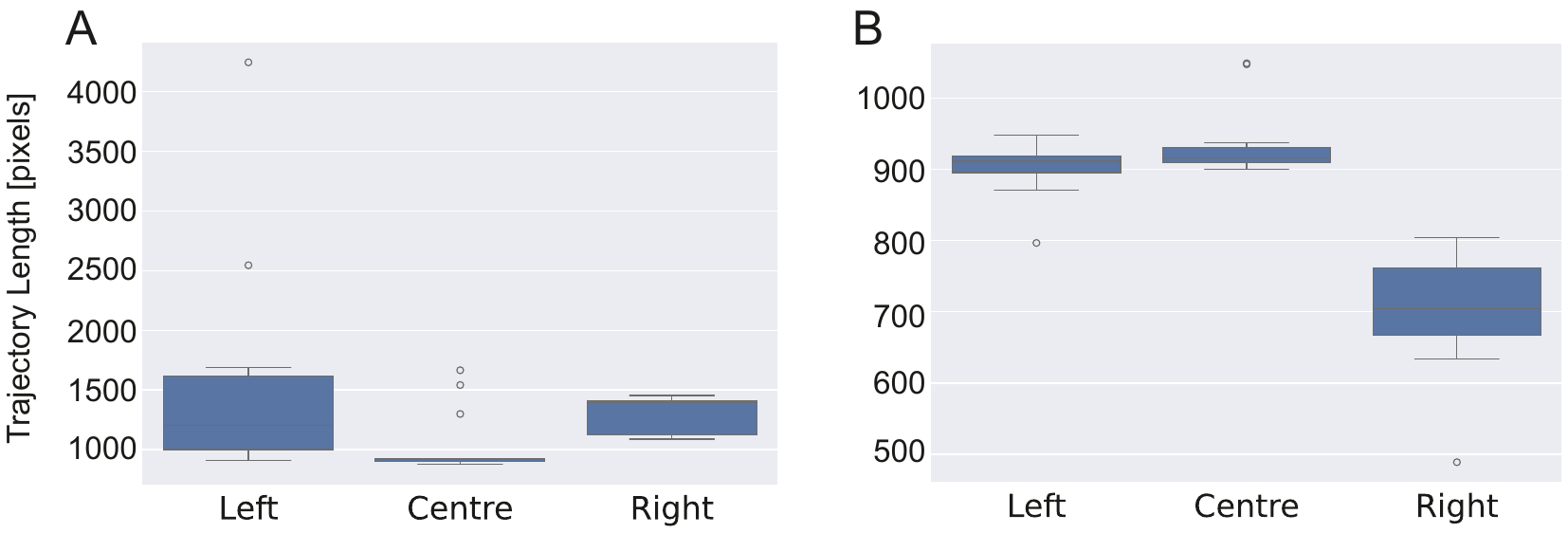}
    \Description{Fig~\ref{fig:figure10} showing trajectory lengths grouped by starting position for each method. A) is for the baseline method. The lowest average trajectory length is for the centre position which also shows the least variance.  The left position is has the next highest average and most variance.  The right position has the highest average but the variance is between the left and centre positions.  Figure 15 B is for the model checking approach.  The centre position has the highest average and least variance but the left position is very similar.  The right position has significantly lower average than the other positions and shows more variance.}
    \caption{Trajectory length grouped by starting position. A: baseline, B: model checking.}
    \label{fig:figure10}
\end{figure}

For the model checking approach, the Kruskal-Wallis test also indicated that there was a significant difference between starting positions $H(2) = 30.2, p < .001$. In this case, post-hoc analysis of pairwise comparisons indicated that there was a significant difference between the right starting position and both the centre ($p < .001$) and left ($p < .001$) starting positions, while the left starting position was not significantly different from the centre starting position. The results suggested that our planning method generates different trajectory lengths when entering the cul-de-sac at an angle of $\approx 45$ degrees from different starting positions, whereas the centre and left starting positions generate similar trajectory lengths.  Note that the shorter trajectories observed when starting from the right position were a consequence of the right veer on our robot, combined with less randomness in the generating process, which is consistent with what we would expect.  Furthermore, when starting from the left  position, the robot would first encounter a disturbance at the bottom of the cul-de-sac (bottom wall in Figure \ref{fig:figure8}A) due to this right veer, resulting in trajectory lengths similar to the centre position, which also first encountered a disturbance at the bottom of the cul-de-sac. Figure \ref{fig:figure10}B shows the distributions of trajectory lengths for each starting position using our approach.


\subsubsection{Collisions, Latency and Memory Usage}\label{sect:mem1}



A total of $3$ collisions were observed for the baseline method across $3$ different runs and $0$ collisions were observed for our model checking approach.  
Table \ref{tab:collisions} shows collisions for the baseline method grouped by starting position. Processing latency was also collected for our planning approach using an in-built time and date package in C++.  The minimum latency for our method was $5.58$ milliseconds and the maximum was $18.76$ milliseconds. 
The mean  was $10.1$ milliseconds with an estimated $95$\% confidence interval of $[9.33, 10.1]$.  

\begin{table}[!h]
\caption{Distribution of collisions for baseline.}
\centering
\begin{tabular}{@{}ll@{}}
\toprule
\textbf{Starting position} & \textbf{Collisions} \\ \midrule
Centre                     & 1                   \\
Left                       & 2                   \\
Right                      & 0                   \\ \bottomrule
\end{tabular}
\label{tab:collisions}
\end{table}

Memory usage was estimated for both model checking and the entire process. 
Model checking memory usage was estimated by summing the compile time memory of the stack, set and adjacency list data structures used by f-DFS and their worst case usage at runtime, based on a state size of 4 bytes, as each state is represented by an integer.  Process memory usage was collected using the Linux top utility at a sampling rate of 1 millisecond. Note that the output includes the amount of shared memory available to a process, not all of which is available to the program (i.e., includes memory that could potentially be shared with other processes), so there is some uncertainty in the measurement. 
The maximum memory usage for model checking at runtime was $372$ bytes, while the maximum memory usage for the entire process was $1.06 - 4.31$ MB. 

\subsection{Scenario 2 - Playground}\label{sect:playground}

Figure \ref{fig:playground_sketch} shows our playground scenario, a closed environment with a single free-standing, static object and a cul-de-sac.  We conducted two comparisons between the baseline method and our model checking approach for a playground environment, focusing on evasion of the cul-de-sac element and collisions. Our aim here was to investigate whether our approach showed improvement in obstacle avoidance behaviour during general operation and gain insight into its versatility.    




\begin{figure}[!h]
    \centering
    \includegraphics[width=0.7\linewidth]{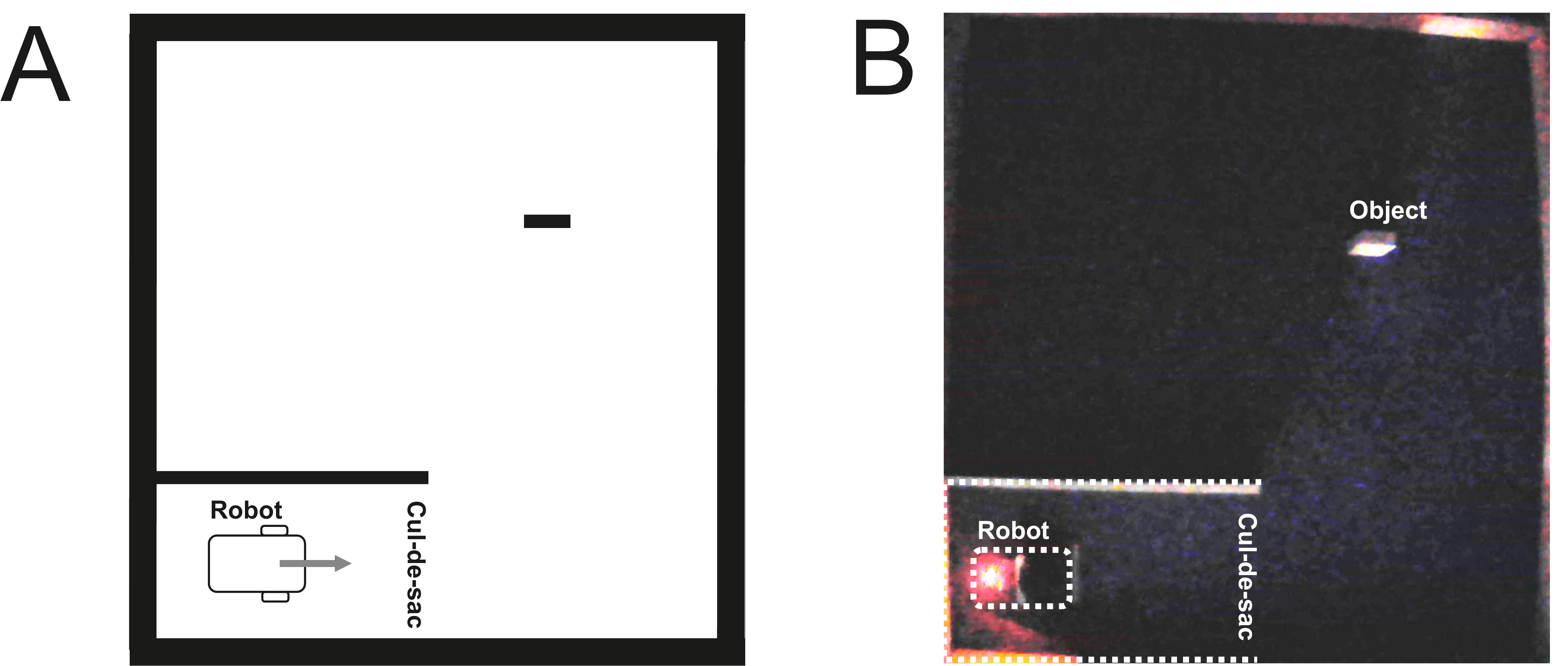}
    \Description{Figure~\ref{fig:playground_sketch}. A: shows a schematic of the playground environment from above.  It is a square closed environment with a cul-de-sac in the bottom left corner and one free-standing object in the top right. A picture of the robot indicates the starting position, which is in the cul-de-sac.  B: shows a photograph of the actual setup in the lab which is the same.}
    \caption{A: shows a schematic of the playground environment from above.  It is a square closed environment with a cul-de-sac in the bottom left corner and one free-standing object in the top right. A picture of the robot indicates the starting position, which is in the cul-de-sac.  B: shows a photograph of the actual setup in the lab which is the same. The actual robot and the cul de sac have a dotted line superimposed to compensate for the low contrast of the real image.}
    \label{fig:playground_sketch}
\end{figure}



\textit{Design of Analysis}. For each run, we calculated the number of times the cul-de-sac was revisited, the time elapsed inside the cul-de-sac and summed the number of collisions, though our analysis was primarily qualitative. We grouped processing latency by the number of steps in the plan and calculated descriptive statistics, estimating $95$\% confidence intervals for each plan length. 

\textit{Data Collection and Preparation}.  
For each run, we let the robot explore the playground environment for 5 minutes and recorded the behaviour on video.  We then extracted the trajectories using optical flow in OpenCV and inspected the difference visually.  In addition, we calculated the number of times the robot visited the cul-de-sac during a run and the time elapsed while inside, using the light on the robot as a point of reference for consistency with analysis of the previous scenario. We also considered the cul-de-sac visited if the robot approached but manoeuvred to evade. 


\begin{figure}[!t]
    \centering
    \includegraphics[width=0.9\linewidth]{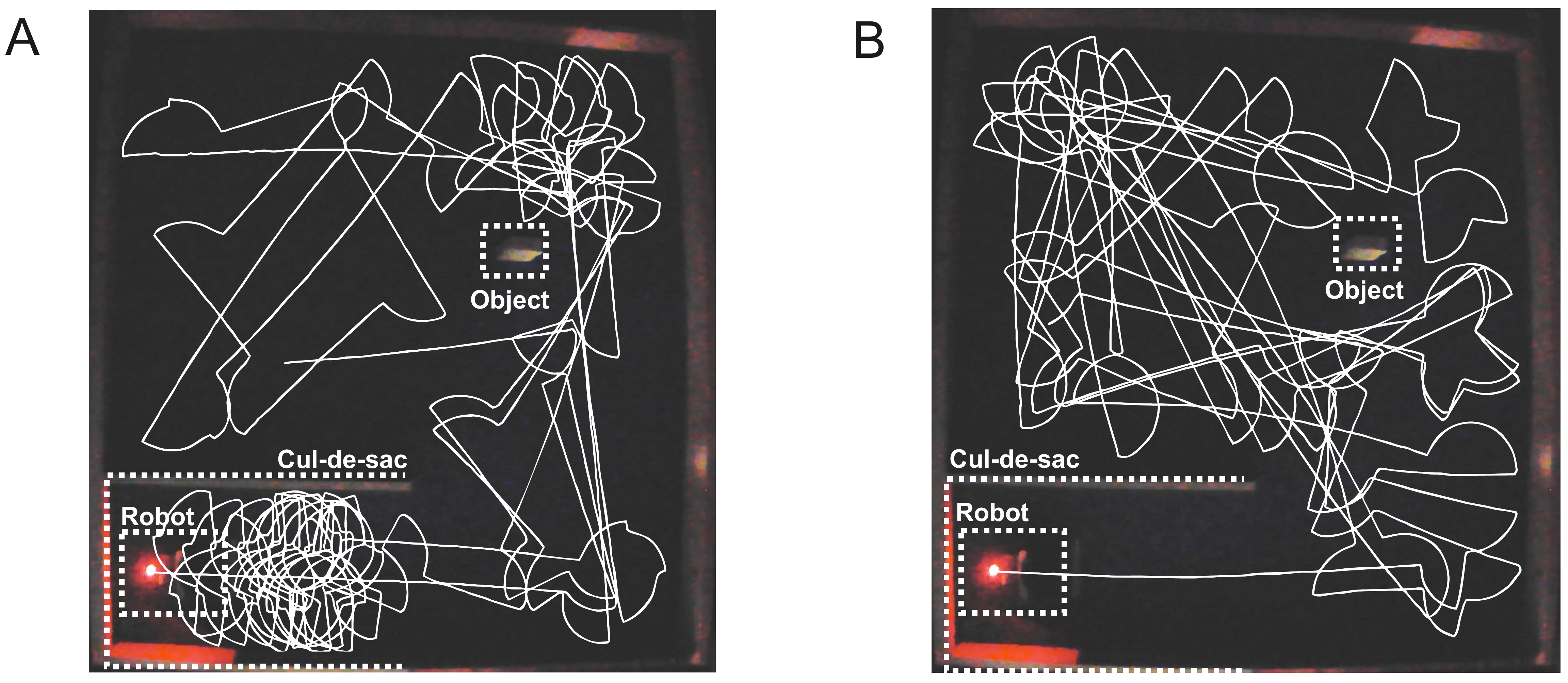}
    \Description{Fig.~\ref{fig:playground1} showing trajectories for comparison 1. A: shows the trajectory of the robot for the baseline method which gets stuck in the cul-de-sac.  B: shows the trajectory of the robot for the model checking approach which does not get stuck in the cul-de-sac.}
    \caption{Comparison 1. Dotted lines indicate robot, cul-de-sac and object. A: the trajectory of the robot for the baseline method which gets stuck in the cul-de-sac.  B: the trajectory of the robot for the model checking approach which does not get stuck in the cul-de-sac.}
    \label{fig:playground1}
\end{figure}

\textit{Results}. Figure \ref{fig:playground1} shows extracted trajectories for the first comparison.  It can be seen in Figure \ref{fig:playground1}A (baseline method) that the robot spends a significant amount of time in the cul-de-sac after leaving it, whereas in Figure \ref{fig:playground1}B (our model checking approach) the robot spends no additional time.  The baseline method caused the robot to revisit the cul-de-sac once at which point it got trapped for $89$ seconds.  There were $4$ collisions for the baseline method and $0$ collisions for our model checking approach.

\begin{figure}[!h]
    \centering
    \includegraphics[width=0.9\linewidth]{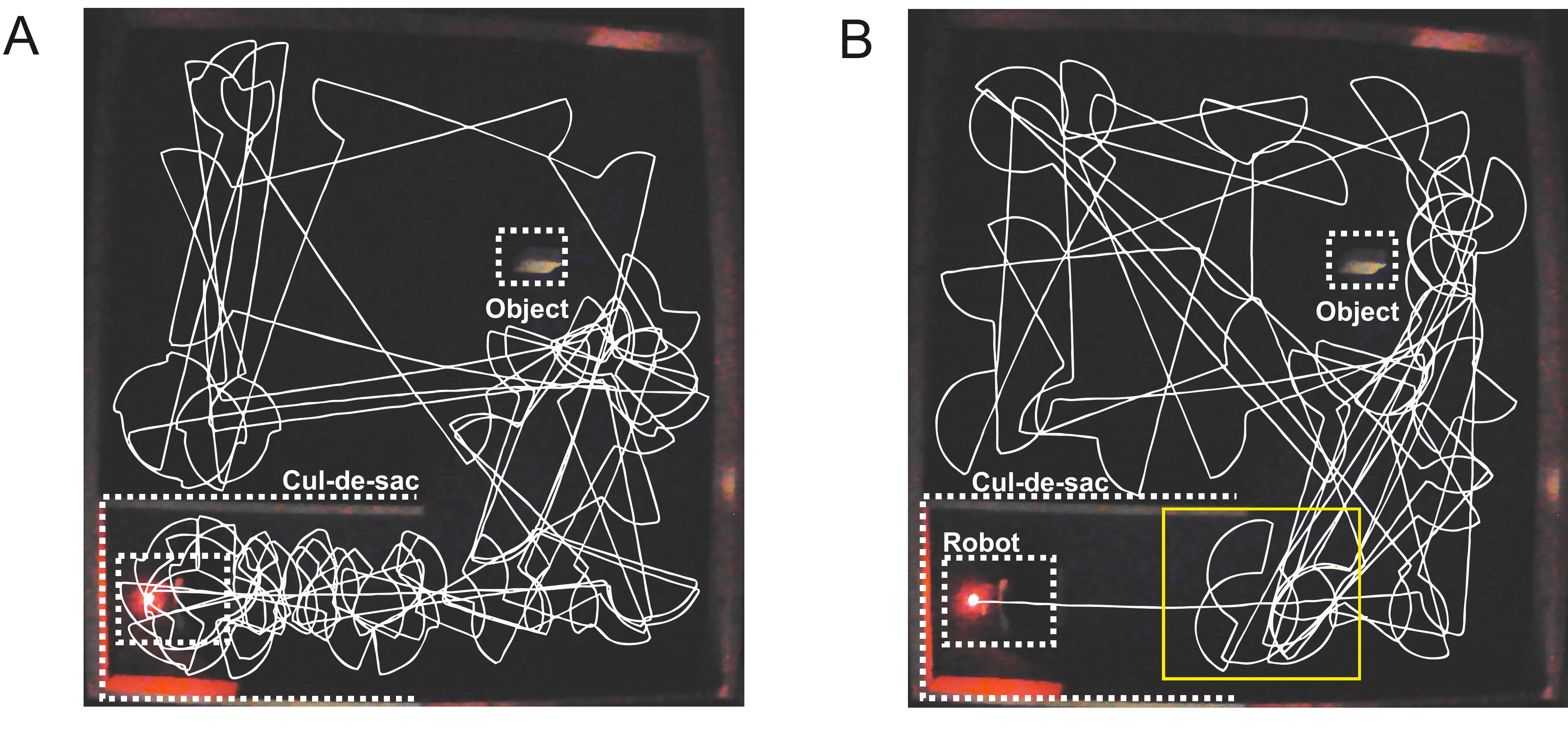}
    \Description{Fig.~\ref{fig:playground2} showing trajectories for comparison 2. A: shows the trajectory for the baseline method which gets stuck in the cul-de-sac.  B: shows the trajectory for planning using model checking which does not get stuck in the cul-de-sac.  A yellow rectangle highlights manoeuvres made by the robot to avoid the cul-de-sac during the run. }
    \caption{Comparison 2. A: the trajectory for the baseline method which gets stuck in the cul-de-sac.  B: trajectory for planning using model checking which does not get stuck in the cul-de-sac. A yellow/solid rectangle highlights manoeuvres made by the robot to avoid the cul-de-sac during the run. }
    \label{fig:playground2}
\end{figure}

Extracted trajectories for the second comparison are shown in Figure \ref{fig:playground2}.  Figure \ref{fig:playground2}A again shows that the baseline method resulted in the robot spending a significant amount of time in the cul-de-sac after leaving it, however in this case the robot revisited the cul-de-sac on three  occasions for a total elapsed time of $79$ seconds.  With our model checking approach, the robot also revisited the cul-de-sac on three separate occasions but evaded rather than entered the cul-de-sac.  The total time elapsed for evasion was $12$ seconds. Evasion manoeuvres are indicated by the yellow box in Figure \ref{fig:playground2}B. On one occasion, the robot had to execute three consecutive three-step plans before a successful evasion of the cul-de-sac, indicating the possibility that there may be scenarios where the robot can still get trapped, though this was not observed for corners.  

\subsubsection{Collisions, Latency and Memory Usage}\label{sect:mem2}

There were $6$ collisions for the baseline method and $0$ collisions for our model checking approach.  Table \ref{tab:processing} shows results for processing latency across both model checking runs grouped according to the number of steps in plans.  There were $42$ two-step plans, $32$ three-step plans and $18$ four-step plans.  All plans were two-step plans in the cul-de-sac scenario, so the effect of plan length on latency was unclear.  However, in this scenario processing latency appears to increase as the length of the plan increases, which makes sense algorithmically.

 As shown in Algorithm \ref{alg:planningalg}, we terminate planning as soon as it is possible to generate a plan of a given length, so the entire procedure is not always executed. Note in Table \ref{tab:processing} that the mean latency for two- and three-step plans are only marginally different. This is because both of these plans are generated immediately after constructing the lateral partitions $P_L$ and $P_R$ using the procedure in Algorithm \ref{alg:lateralalg}, which has a time complexity of $O(n)$ where $n$ is the resolution of the LiDAR scan.  As the resolution is fixed for the hardware, the worst case performance of Algorithm \ref{alg:lateralalg} is invariant for the robotic platform. Prior to assessing whether a two- or three-step plan is appropriate for the situation, we also translate the observations (see Algorithm \ref{alg:planningalg}) which is again linear on the resolution of the LiDAR scan. All other operations up to assessing whether a two- or three-step plan is possible are in constant time, so the only remaining difference is in search using f-DFS which has a time complexity of $O(|V| + |E|)$ where $|V|$ is the number of vertices and $|E|$ is the number of edges in   the underlying graph. For two-step plans, in the worst case 15 states are searched in the model, as the accepting state is not a leaf node. The worst case for three-step plans is 9 states.\label{loc:r2c3 page}  

\begin{table}[!h]
\caption{Processing latency in milliseconds for plans consisting of 2, 3 and 4 steps.}
\centering
\begin{tabular}{lllll}
\toprule
\textbf{Steps($n$)} & \textbf{Min.} & \textbf{Max.} & \textbf{Mean} & \textbf{CI (95\%)}  \\ \midrule
2              & 2.5           & 14.09         & 6.6          & {[}5.92, 7.38]      \\
3              & 5.12          & 12.15         & 7.56          & {[}7.04, 8.23]      \\
4              & 5.46          & 21.62         & 11.49         & {[}9.6, 13.82] \\  \bottomrule 
\end{tabular}
\label{tab:processing}
\end{table}

The major increase in mean processing latency shown in Table \ref{tab:processing} is for four-step plans.  Progressing to this stage in Algorithm \ref{alg:planningalg} requires constructing longitudinal partitions using Algorithm \ref{alg:longalg}, which again has a time complexity of $O(n)$.  In the worst case, f-DFS will search the entire 15 states of the model when generating a four-step plan.  However, it should be noted that we construct two lateral partitions (two- and three-step plans) and four longitudinal partitions (four-step plans). As our model is small and has a fixed size, the almost doubling of mean processing latency observed for four-step plans is likely due to Algorithm \ref{alg:longalg} and construction of the longitudinal partitions.   


We used the same method as for the cul-de-sac scenario in Section \ref{sect:culdesac} for estimating the memory usage. Maximum memory usage by model checking at runtime was similar, calculated at $360$ bytes, and maximum memory usage of the process was $1.06 - 4.09$ MB.



\subsection{Summary}\label{sect:summary1}

We conducted a close study of a cul-de-sac in Scenario 1.  We compared trajectory lengths and collisions for a baseline method and our model checking approach and found evidence that our method produces more efficient trajectories for avoiding a cul-de-sac.  There were 3 collisions for the baseline method and 0 for model checking.  We also collected data on the processing latency of our approach.   Processing latency was well within the $100$ milliseconds real-time deadline. However, we found that the nature of the cul-de-sac setup had side effects which restricted the generality of our conclusions---only three-step plans were generated, so we had no data on processing latency for the range of plan lengths. Memory usage by our model checking approach was $372$ bytes, while the maximum memory usage of the entire process was $1.06 - 4.31$ MB.


We let the robot roam free in a playground environment in Scenario 2.  We compared the time spent in the cul-de-sac and the number of collisions between our approach and the baseline method for two comparisons.  In both comparisons,  the robot spent a significant amount of time in the cul-de-sac after leaving it using the baseline method and no time revisiting the cul-de-sac using our approach. However, it was noted that there may be some edge cases where our approach could result in the robot getting trapped. In total, there were 6 collisions for the baseline method and 0 for our method.  We collected data on latency which was grouped according to plan length, providing some experimental insight into the effect of plan length.  This was associated with an increase in mean processing latency. However, the maximum latency was $21.62$ milliseconds, again well within the 100 milliseconds deadline. The maximum memory usage by model checking was similar to Scenario 1, calculated at $360$ bytes, with maximum memory usage of the process at $1.06 - 4.09$ MB.  


\section{Discussion}\label{sect:related}



We have addressed \textbf{RQ} (defined in Section \ref{sect:eval}) by showing that a bespoke implementation of model checking can be used to reliably plan sequences of closed-loop tasks for real-time obstacle avoidance on low-powered robots in static environments, assuming a bounded sensor noise model and real execution error within acceptable tolerance levels. In contrast to a baseline method, which is can only plan one step ahead, our application of real-time model checking reliably avoids collisions and is not prone to getting trapped in a cul-de-sac. Our approach achieved a processing latency within hard real-time constraints and used minimal additional memory.

We have deliberately focused solely on obstacle avoidance as an initial proof of concept. However, the aim is to integrate our approach into hierarchical decomposition (HD) frameworks. This would lead to an extension of HD beyond pure reactive obstacle avoidance. In  \cite{Cai_Yan_Shi_Zhang_Guo_2023}, for example,  the planning task is decomposed into sequences of obstacle avoidance and targeting tasks. While the low level obstacle avoidance tasks in \cite{Cai_Yan_Shi_Zhang_Guo_2023} are strictly reactive, their model could be extended to better negotiate U-shaped obstacles by integrating our look-ahead approach, allowing for the reliable avoidance of situations like cul-de-sacs. Decompositions into sub-goals \cite{Triolo2020} are also used in receding horizon paradigms \cite{Wang2026}; our approach could be integrated here too.

A general problem of reactive agents, such as Braitenberg vehicles, is that they get stuck in concave objects, such as cul-de-sacs or corners. To prevent Braitenberg vehicles ending up in these fatal situations \cite{Gogoi2022} application specific corner detectors are added to trigger corrective action. In contrast, our approach does not require feature detection because it uses direct LiDAR detections surrounding the robot. In general, any planning algorithm that uses chained attraction and avoidance behaviour \cite{Gogoi2022,Cai_Yan_Shi_Zhang_Guo_2023} can benefit from our model checking approach because our algorithm will make it substantially less likely for a robot to get stuck while avoiding obstacles. Examples include 
autonomous cleaning robots that usually operate reactively \cite{Corke2022}, 
and industrial line-following robots when
encountering unforeseen obstacles.

There have, however, been other notable implementations of real-time model checking for planning on autonomous systems. In \cite{lehmannneedles2021} an approach for \textit{online strategy synthesis} was developed to derive action plans with formal safety guarantees for steerable needles implemented using the \textsc{Uppaal Stratego} \cite{baier_uppaal_2015} model checker and Python. Plan synthesis took between $0.04$ and $6.99$ seconds. Our approach was reliably faster as we did not use an off-the-shelf model checker so were able to optimise for real-time. Standard model checkers include additional functionality leading to a bloated system appropriate for offline verification but sub-optimal for online planning tasks, especially on a resource-constrained device. Indeed, this issue was encountered in \cite{pagojus_simulation_2021}, where time lag due to model compilation (approx. 3 seconds) made using the SPIN model checker infeasible for real-time planning on a real autonomous robot  particularly in high speed contexts \cite{li_real-time_2016}. 

Note that the environment in \cite{lehmannneedles2021} was structured and fully observable. This assumption is often unrealistic when considering navigation for autonomous robots \cite{Luckcuck2019}. Driving in the real world, for example, is a broad activity domain in which complete priors are impossible; developers simply cannot foresee all future situations \cite{chollet_measure_2019}.   Consequently, high definition maps are often used in AVs. These are notoriously difficult to maintain and imprecise in complex driving conditions \cite{wong_mapping_2021}.  In SAE level 5 autonomous driving (i.e., full automation requiring no human oversight) an AI-enabled driving system should therefore be able to adapt to unpredictable situations based on egocentric sensory inputs.  This requires the ability to understand and interpret structural information about the environment and reason flexibly about possible courses of action which our method does. 

This is in contrast to deep learning approaches for autonomous navigation, such as deep reinforcement learning (DRL), which can only learn rigid generalisations from data \cite{sun_motion_2021}.  As a consequence, they ignore structural information from the immediate environment vital for reactive obstacle avoidance.  While deep neural network approaches are attractive due to the possibility of doing end-to-end learning (i.e., learning optimal actions directly from sensor data), even simple models require large amounts of training data, yet they remain prone to error in unstructured environments and lack the ``common sense'' necessary to reliably adapt to unpredictable surroundings. In autonomous driving, the  generalisation error of popular DRL approaches is further compounded by the so-called ``reality gap'' \cite{chen_end--end_2024, hu_how_2023}.  For ethical reasons, models are trained offline in simulated environments. While this generates large amounts of training data, simulation ultimately cannot replicate real-world driving conditions, resulting in a model which serves predictions in real-time but makes opaque decisions insensitive to the local environment, increasing the risk of collisions.

Online methods using symbolic representations (like ours) can potentially offer a safe, reliable, and explainable route to flexible reasoning for partially observed, unstructured environments. However, real-time performance is crucial, particularly within a high speed driving context---it has been argued that a lag of even 100 milliseconds is unacceptable \cite{lin_architectural_2018, li_real-time_2016}.  As DRL approaches are typically fast once trained, little attention has been paid to symbolic methods for real-time obstacle avoidance.  Some research has investigated them in partially observed and unstructured environments.  In  \cite{Bouraine2012}, for example, runtime verification was used to ensure that a robot is at rest when a collision occurs in order to minimise damage. However, no planning was involved and real-time performance was reported at $49$ milliseconds (via simulation on an average laptop).  Our approach achieves real-time planning with a processing latency of less than $22$ milliseconds on a low-powered robot. 

Indeed, the inability to plan is a general limitation of runtime verification methods, which typically synthesise a monitor to flag unwanted behaviour during execution and potentially trigger remedial action.  In \cite{Desai2017}, for example, drones were simulated offline, plans generated by conventional means, and runtime monitors synthesised using model checking to ensure model properties are satisfied during execution.  In both \cite{Pek2020} and \cite{Liu2017} a safety monitor was synthesised to provide an extra layer between the motion planner and trajectory tracker. The monitor checked whether trajectories 
violated certain criteria. 
In \cite{DALZILIO2023}  a monitor was  synthesised based on verified models of a quadcopter flight controller to trigger corrective action (e.g., emergency landing). 
We perform planning with the model checker itself instead of checking if an existing algorithm leads to a dangerous situation \cite{Gu2021}, and define a robot safe zone to ensure that our approach does not lead to collisions during execution.

Planning algorithms often assume that a robot moves in a global coordinate system 
in which trajectories are generated as continuous or piecewise continuous curves.
In this context, model checking can be performed offline to synthesise a runtime monitor for the trajectory following error and signal fatal deviations \cite{Lin2022}, effectively enabling  classical line following.   
However, converting a local coordinate system into global coordinates from unreliable sensor data is non-trivial \cite{Leonard1991, Smith1986} and hard to achieve in real-time without precise offline map data. In \cite{Pek2020}, for example, a cleaned and labelled global map is used to check  generated trajectories for potential violations, and in \cite{Liu2017} the locations  of pedestrians in global coordinates are simulated. 
In contrast to these approaches, we take a constructivist \cite{porr_inside_2005} approach and keep the environment partially observable by only working with sensory inputs that are accessible from an egocentric perspective, i.e., the position of disturbances relative to the robot.
A global perspective on the robot is not required as we do not perform precise line following along a globally defined trajectory. Instead, closed-loop tasks have an attentional focus on a specific disturbance which they need to counteract. 

The use of a global coordinate system, however, is common in robotics.  Model Predictive Control (MPC) \cite{Stano2023} predicts future states based on a system model which describes the dynamics of the robot in its environment and can incorporate safety requirements as constraints, often formalised with Control Barrier Functions (CBFs). MPC-based approaches define a cost function, usually based on navigation errors obtained from the robot while operating in its environment. A representative example related to our approach is \cite{Zhang2023} which investigated planning at a traffic junction. In \cite{Zhang2023} the coordinate system is 
global and the dynamical model 
directly creates the robot trajectory in global coordinates. 
However, a real robot has no direct access to a global coordinate system, only egocentric coordinates in the  local reference frame. As mentioned previously, transforming from one frame to the other is non-trivial, prone to error, and may require another PID controller to minimise the error between the two systems. Such real-world practicalities can be avoided by  running MATLAB simulations of robots, so it is rare to see MPC on real robots in the literature. In \cite{Hong2025}, for example, MPC is used for lane assist where the cost function reflects the deviation of a vehicle from the centre of the lane, with CBFs used to formally define safe and unsafe states of the system. The cost function is defined as the level of intervention of the freely running dynamical system by corrective steering, braking, and acceleration from the driver. In order to minimise the cost function both the dynamics of the system and the CBFs need to be known. However, this is often difficult in real scenarios and completely avoided in \cite{Hong2025} by evaluating the method via simulation in MATLAB and the CarMaker simulation environment within Simulink---global trajectories were generated within the simulation which were then transferred to the robot. In contrast, our approach does not create trajectories, rather it is firmly rooted in the egocentric coordinate system of the robot. 
After using a looming disturbance as a reference location, local information is used to generate a sequence  of discrete actions to avoid the disturbance. By chaining states behaviours can be executed which would otherwise be hard to achieve with a dynamical system. This includes, for example, performing a U-turn or driving around a corner.   

Another canonical approach to planning for robots is Rapidly expanding Random Trees (RRT, RRT*) \cite{LaValle1998Rapidly-exploringPlanning,Karaman2011} and its predecessor Probabilistic Roadmaps (PRM) \cite{Kavraki1996}, both efficient methods for solving planning problems in high dimensional Euclidean spaces.
They are said to be "sampling based" as they randomly explore adjacent states to traverse the state-space towards a target state, avoiding the necessity to store all states. Computational cost is further reduced in \cite{Wang2020} in which a deep neural network is used. Our current robot navigation task does not require such an approach as our state-space is small. This has the benefit of avoiding jerky movements and the need for post smoothing of trajectories \cite{kuffner2002} to which sampling-based approaches are prone. 
Further limitations of \cite{LaValle1998Rapidly-exploringPlanning,Karaman2011} and \cite{Wang2020}, like the MPC approaches previously described, 
include the use of a global coordinate system combined with simulation to provide a global perspective. Again, navigation was performed in simulation, where the precise global coordinates of the robot, obstacles, and target were known.

Similar to our method, the lattice-based approach of \citet{dhar_robust_2024} incorporates a model of expected disturbance-to-trajectory-execution into the control action, where disturbance rejection is operated online via a feedback controller. Disturbances in this approach  are defined a priori as regions surrounding the ideal motion trajectory of the robot. The approach was again tested in simulation using a global coordinate system. In a local scenario, the method would require not only previous knowledge of potential disturbances affecting the trajectory of the robot, but also precise global localisation, which is computationally expensive and difficult to achieve.  

It is argued in \cite{Lozano1984} that 
robotic planning only requires knowledge of the relationships between objects in the environment, 
rather than the precarious global knowledge of the environment typically required for generating exact global trajectories. Our approach in this paper is similar: we identify a disturbance and use it as a focal point for decision-making, focusing on the \textit{relationship} between the robot and the disturbance. However, unlike  \cite{Lozano1984}, where it is 
argued that low level continuous control and decision-making can be completely separated, we do consider the execution dynamics of our system in our method for the generation of plans.  

Similar to our approach,  the research in \cite{Pivtoraiko2009}
chains controllers to make plans and validates their planning method on a real robot instead of simulation.  A hybrid system using adaptive sampling of the environment is used to create the chained controllers. 
Similar to this paper, they focus on cul-de-sac scenarios, as U-shaped obstacles are the most problematic for mobile robots
and highly relevant in real world applications.
The robot in \cite{Pivtoraiko2009} can enter a cul-de-sac but must then devise a plan to escape, whereas we use model checking to devise a plan that avoids the robot entering the cul-de-sac in the first place. 
Their approach also involves the generation of a sophisticated dynamic grid. We argue that this is both unnecessary and expensive. We instead use local snapshots to develop immediate plans, which are then discarded, avoiding explosion of the state-space. 

Indeed, one of the main advantages of our approach is that our state-space is small, due to breaking down the immediate environment into spatial primitives representing task execution and chaining the represented controllers. A similar approach to tackling state-space explosion within the context  of model checking for planning is also investigated in \cite{Wongpiromsarn2011}, where the TuLip toolbox is used for the synthesis of controllers (e.g., for autonomous cars). State-space explosion is mitigated by applying a receding horizon and breaking tasks down (like we do).
Unlike our approach, however, their chained motion controllers are fully specified dynamical systems which is impractical in many real-world applications where precise environmental transfer functions are often not known.  

The use of model checking for real-time planning in real robots is rare, however it has perhaps been achieved in \cite{Sarid2023}, where plans are described using LTL whose states represent motion primitives.
This is a typical hybrid approach to planning and control that is similar to ours, though states in our approach  represent possible future configurations of the robot, relative to a local frame of reference. In addition, the robot in \cite{Sarid2023} uses a global map to navigate which is created by translating the egocentric LiDAR coordinate system into global coordinates provided by a Vicon motion system. Our approach in comparison is egocentric and does not require a global coordinate frame.

An obvious limitation of our method is that obstacle avoidance is not target-driven.  
We have attempted to extend the approach in this paper to accommodate target-driven behaviour in preliminary studies, which of course required some form of localisation. We used dead reckoning for localisation, however, while computationally lightweight, this was too imprecise for robust obstacle avoidance.  While we were able to overtake an obstacle to arrive at a target, there was a lot of variance in the distance from the target due to accumulated error.  Our approach involved having two models working hierarchically in relation to each other, each with its own path property, one for tracking the goal and another model for avoiding obstacles, similar to the model in this paper but smaller. However, we found that this approach was brittle and prone to failure.  We concluded that it would be better to approach the problem afresh rather than extend the approach taken in this paper.  One promising idea is to use a multi-objective property combining safety and progress and utilise SPIN's iterative search for short trails for more flexible plan specifications in LTL. 


As an explainable  method, we believe our approach shows promise as an avenue for the development of safe and reliable decision-making for AVs.  In addition, formal verification techniques for hybrid systems could potentially be integrated with our planning approach to provide safety guarantees on behaviour at design time.  In \cite{kamali_formal_2017}, for example, vehicle platooning is represented as a multi-agent system using the \textsc{Gwendolen} agent programming language. Verification is modularised using Agent Java PathFinder (AJPF) \cite{dennis_model_2012}  for discrete decision-making and the \textsc{Uppaal} model checker \cite{hendriks_uppaal_2006} for continuous real-time requirements. A similar approach was used to verify the behaviour of a rational agent acting as decision-maker for an AV in \cite{fecualfi2017}. In general, verification of hybrid systems is difficult.  However, a modular approach combined with the use of ``core'' knowledge may simplify the effort, especially if the method has been designed to be verifiable.

\section{Conclusion}\label{sect:conc} 

Our approach is a successful first step towards the use of model checking for real-time planning.  The main benefit of our approach is the ability to reason flexibly about the immediate environment based on egocentric sensory inputs, making it reactive to structural information necessary for avoiding imminent collisions.  In future work, we intend to extend our approach to accommodate target-driven behaviour while continuing to meet hard real-time constraints. We believe that our approach using model checking shows promise as an avenue for the development of safe, reliable and explainable decision-making for autonomous vehicles which is also energy efficient.

\begin{acks}
For the purpose of open access, the author(s) has applied a Creative Commons Attribution (CC BY) licence to any Author Accepted Manuscript version arising from this submission. 

This work was supported by a grant from the \grantsponsor{governance}{UKRI Strategic Priorities Fund to the UKRI Research Council}{https://www.ukri.org/} [\grantnum{governance}{EP/V026607/1}]; the \grantsponsor{socialcdt}{UKRI Centre for Doctoral Training in Socially Intelligent Artificial Agents}{https://www.ukri.org/} [\grantnum{socialcdt}{EP/S02266X/1}]; and the \grantsponsor{epsrc}{UKRI Engineering and Physical Sciences Research Council Doctoral Training Partnership} a award [\grantnum{epsrc}{EP/T517896/1-312561-05}]. 



\end{acks}

\appendix
\renewcommand\thefigure{\thesection.\arabic{figure}}
\setcounter{figure}{0} 

\section{Testing Our Method Implementation on Different Hardware}\label{app:cross}

Our planning method was designed to be cross-platform, so to validate that our approach  indeed works on different hardware, we executed some additional experimental runs on a second robot. As for our original robotic platform,  source code and instructions are available on Zenodo \cite{chandler_2026_21000117}.

\subsection{Alternative Robotic Platform}\label{app:deltabot}

The second robot is shown in Figure \ref{fig:deltabot_dims}.  To ensure our approach still worked with a different laser range scanner, we equipped the robot with an RPLiDAR C1 by Slamtec\footnote{\url{https://www.slamtec.com/en/c1}}, and for actuation we again used two continuous rotation servos by Parallax\footnote{\url{https://www.parallax.com/product/parallax-continuous-rotation-servo/}}.  The hardware programming interface for the robot was a Radxa Rock 5B\footnote{\url{https://radxa.com/products/rock5/5b/}} single board computer running both a Quad Core Cortex-A76 at 2.2---2.4GHz and a Quad Core Cortex-A55 at 1.8GHz with 4GB RAM and wireless LAN. 

\begin{figure}[!h]
    \centering
    \includegraphics[width=\linewidth]{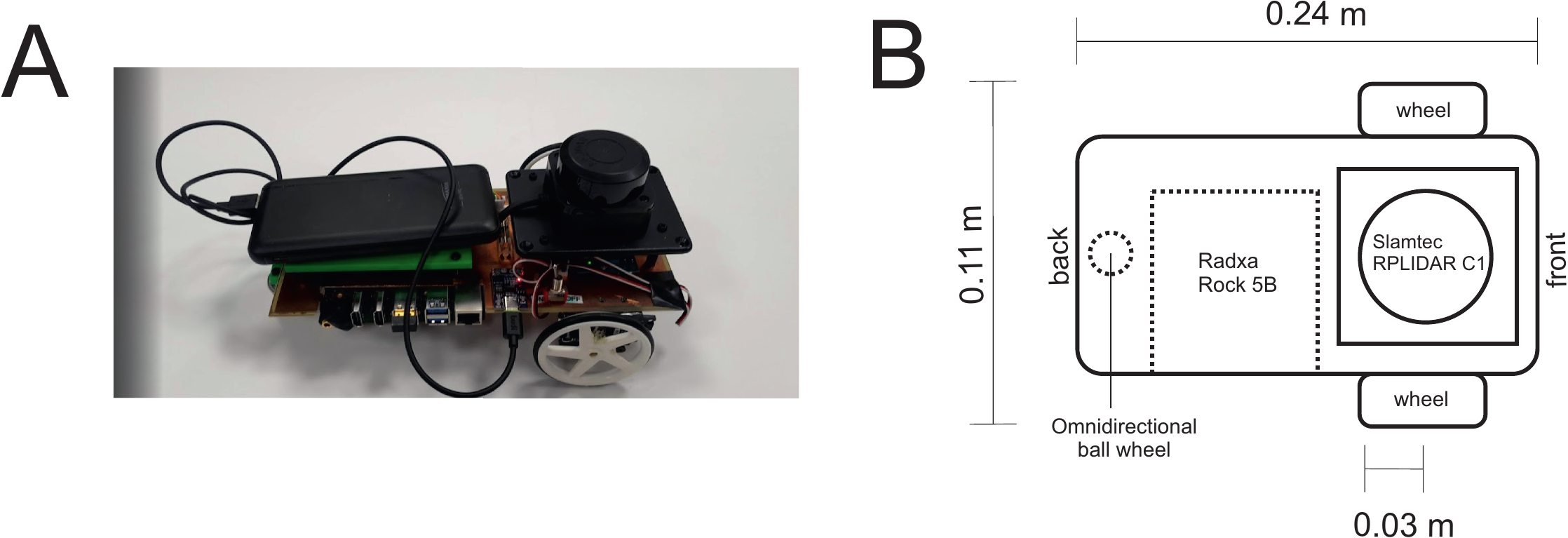}
    \caption{A: our new robotic platform.  B: schematic showing robot dimensions.}
    \label{fig:deltabot_dims}
\end{figure}

The RPLiDAR C1 generates a 2D point cloud at a rate of $\approx10~\mathrm{Hz}$ in contrast to the RPLiDAR A1M8 which only generates a point cloud at a rate of $\approx5~\mathrm{Hz}$. As with our original experiments, we used a 2D vector space to model the point cloud with the origin representing a point at the centre of the LiDAR.  The robot is oriented towards positive $x$ by convention which corresponds to the heading $0$ radians.  Rotating $90$ degrees left is equivalent to the heading $\frac{1}{2}\pi$ radians, and rotating $90$ degrees right is equivalent to the heading $-\frac{1}{2}\pi$ radians. The maximum heading is $\pi$ in the left direction and $-\pi$ in the right, representing a $180$ degrees rotation relative to the local frame.

\subsection{Testing on Cul-de-sac Scenario}\label{app:cross_runs}

To test our approach on the new robot, we used the same cul-de-sac environment setup described in Section \ref{sect:culdesac}.  We conducted nine runs in total with our planning method, with three runs from each starting positions ("Left", "Centre", "Right"). We again extracted trajectory data from videos of the runs using the optical flow method in OpenCV.  Figure \ref{fig:deltabot} shows one example of a run from each starting position.  Note that the light we tracked with the optical flow method was on a different part of the robot in comparison to our original robot (towards the front on the new robot instead of the tail), so rotations for an avoid task $T_{L/R}$ look slightly different in the extracted trajectories. 

\begin{figure}[!h]
    \centering
    \includegraphics[width=\linewidth]{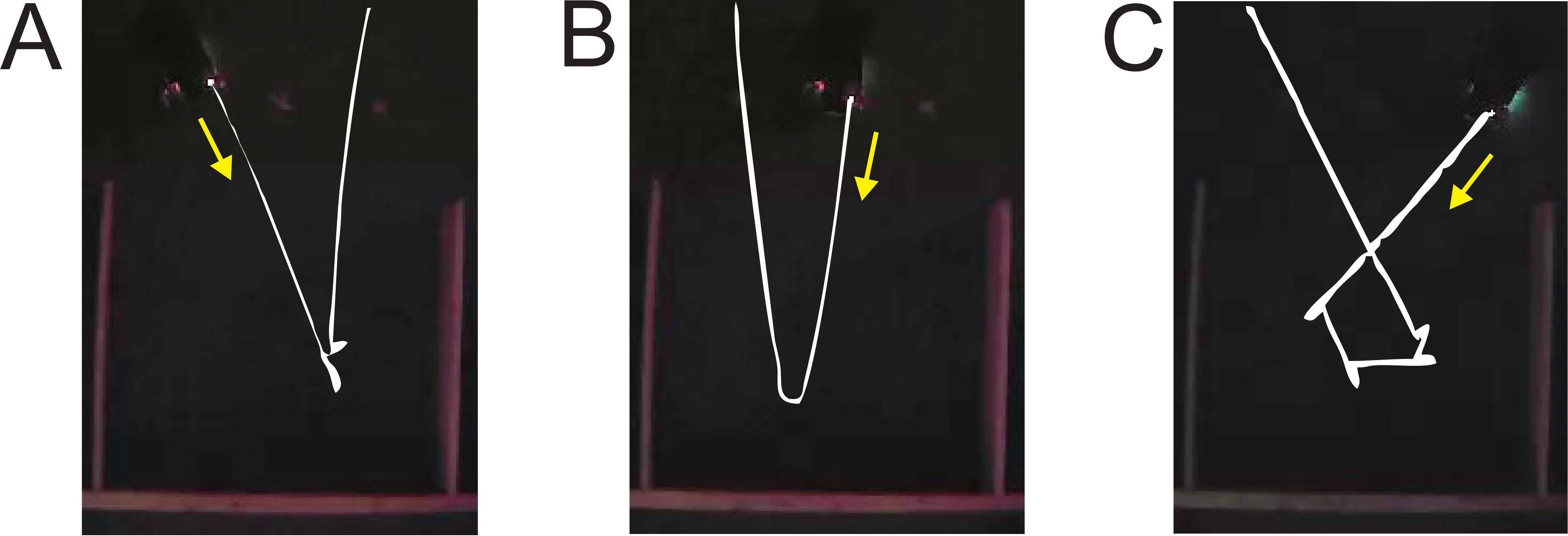}
    \caption{A: the "Left" starting position. B: the "Centre" starting position.  C: the "Right" starting position.}
    \label{fig:deltabot}
\end{figure}

Our method mostly worked across all runs, following some adjustment of the parameters of the abstraction described in Section \ref{sect:sim}, which was expected. However, there was one run from the "Centre" starting position where the robot appeared to  avoid something that was not there, after initially generating a three-step plan to go back the way it came (like the other two runs from the "Centre" starting position). Upon further inspection of the data, "ghost" points were being sensed, which could be due to a low-level issue with the RPLiDAR C1 SDK or presence of dust. This highlights that our approach could be improved to cope better with noisy measurements. A possible solution would be to require a minimum number of LiDAR points to make up a disturbance, or evaluating temporal consistency between scans (i.e., not changing plans if points disappear between consecutive scans). When points are \textit{not} ghost points, this would result in a maximum delay to planning of 10 milliseconds if the temporal consistency window is chosen as two scans.

\renewcommand\thefigure{\thesection.\arabic{figure}}
\setcounter{figure}{0}

\section{Theoretical Results\label{app:theory}}

In this appendix, we give informal proofs of two properties of our approach. We show that for any arbitrary static environment: (i) the robot cannot get trapped in a corner, and (ii) no disturbance can enter the safe zone, as long as the parameters of our abstraction meet certain criteria.

\subsection{Corner Scenarios}

One of the main drawbacks of the baseline method is the potential for the robot to get trapped in a corner, as explained in Section \ref{sect:motivation}.  For this reason, we compared our method to the baseline method in a close study of a cul-de-sac scenario with two corners (see Section \ref{sect:culdesac} for details).   

One definition of ``getting trapped'' is that the robot executes a repeating sequence of tasks which results in the robot visiting the same configurations in $C$-space in a loop. Getting trapped in a corner is one specific instance of this behaviour, defined here as the robot alternating between the avoid tasks $T_L$ and $T_R$. However, this notion of getting trapped assumes that a single avoid task $T_{L/R}$ is sufficient for avoiding a disturbance. If an avoid task does not successfully eliminate a disturbance from its path in the first instance, then we can also say that the robot is trapped, on a corner or otherwise.  In our approach, we require that the robot rotate left or right $\frac{1}{2} \pi$ radians (i.e., orthogonal to the disturbance) which ensures that the front of the robot is subsequently facing  a different direction, regardless of the angle of the disturbance witnessed in $P_{look}$ relative to the robot.  Assuming no second disturbance, the robot can then continue driving straight without a collision.  If in addition the environment is static, there is no execution error beyond acceptable tolerance levels, and sensor data is conformant to a bounded sensor noise model, then it is trivial to show that our method cannot get trapped in a corner, as it is structurally impossible.

\begin{theorem}
   Let $\Lambda = T_0 T_1 T_2 \dots$ be an infinite sequence of tasks from the set of tasks $\mathcal{T}$ defined in Equation \ref{eq:tasks}.  Then there is no subsequence of tasks in $\Lambda$ which alternates between tasks $T_L$ and $T_R$.
\end{theorem}

\begin{proof}
Assume there is a subsequence in $\Lambda$ which alternates between $T_L$ and $T_R$. Then $\Lambda$ contains at least the ordered pair $\langle T_L, T_R \rangle$ or  $\langle T_R, T_L \rangle$. By definition, no path in the disturbance-focused transition system $DTS$ contains $\langle T_L, T_R \rangle$ or  $\langle T_L, T_R \rangle$, so no subsequence of tasks generated as a plan can contain a subsequence which alternates between tasks $T_L$ and $T_R$.  As the last task in a plan is always the default task $T_0$, it is not possible for the last and first task of consecutive plans to form the pair $\langle T_L, T_R \rangle$ or  $\langle T_L, T_R \rangle$. There is no other way that the pairs can be formed. Hence  a subsequence in $\Lambda$ cannot alternate between $T_L$ and $T_R$ which contradicts our assumption.
\end{proof}

Note that while under the stated pre-conditions our model implies that the robot cannot get trapped in a corner, it is nonetheless possible for our robot to get trapped in other ways. For example, in some environments the robot could get stuck in a loop by always turning left/right then driving straight, either by repeatedly executing two-step or four-step plans.  Depending on the angle of the initial disturbance, it is also possible for the robot to get trapped between parallel walls by executing three-step plans repeatedly.  Indeed, we saw some evidence of this in Section \ref{sect:playground}.

\subsection{Safe Zone}  

Assuming that the pre-conditions discussed in the previous section (i.e., static environment, no execution error beyond acceptable tolerance, sensor data conformant to a bounded sensor noise model, etc.) hold, then our use of a safe zone mitigates the need for a formal guarantee that our abstraction preserves safety for an avoid task (i.e., a rotation), as indicated in Section \ref{sect:twostep}.  The safe zone is defined as a circle around the centre of mass of the robot, however in practice we over-approximate the safe zone using a square for simplicity.  This is defined as the partition:

\begin{equation}\label{eq:safe}
    P_{safe} = \{o \in O\ |\ 0 < |o_x| \leq d_{safe}\ \land 0 < |o_y| \leq d_{safe} \}
\end{equation}

\noindent where $0$ indicates the origin of the point cloud, which in the real world correlates with the centre of mass of the robot, and $d_{safe}$ is the radius of the safe zone.  

As mentioned in Section \ref{sect:twostep}, if the robot is executing the default task $T_0$ and a disturbance $D^+$ is witnessed in partition $P_{look}$, then a new plan is generated using model checking. The robot then continues driving straight in task $T_S$ until the disturbance is witnessed in partition $P_{shield}$ at which point the first task in the plan is initiated.  If the task $T_S$ is an intermediate task in  a four-step plan, then the next task in the plan is initiated when a disturbance is witnessed in $P_{shield}$.  As an added fail safe, the robot stops if a disturbance is witnessed in partition $P_{shield}$ and no plan is available.  

To ensure action is always taken to prevent a disturbance from entering the safe zone, we require that the shield partition be defined as:
\begin{equation}\label{eq:shield}
    P_{shield} = \{o \in O\ |\ d_{safe} < o_x \leq v \triangle t + \epsilon\ \land |o_y| \leq \frac{1}{2}(L + tol)\}
\end{equation}

\noindent where $d_{safe}$ is again the radius of the safe zone, $v \triangle t$ is the predicted distance the robot will travel in one closed-loop step, $\epsilon$ is the negative or positive error between the actual and predicted distance, and $\frac{1}{2}(L + tol)$ defines the width of the partition in the lateral dimension. We also require that $\frac{1}{2} (L + tol) = d_{safe}$ to ensure approaching disturbances are always witnessed. 

Note that while no collisions were observed in test, the lateral dimension of the shield partition in our case study was narrower than the safe zone, so a disturbance \textit{could} have entered the safe zone without being witnessed by the shield partition and action taken, potentially leading to a collision with the tail of the robot during a subsequent avoid task $T_{L/R}$.  The reason we opted to make it narrower than $d_{safe}$ was to make it possible for the robot to navigate narrow passages.  Similar to the lateral partitions defined in Equations \ref{eq:poslat} and \ref{eq:neglat}, we are only interested in the nearest observation to the robot.  Consequently, a disturbance is defined as $D^+ = min(o_x \in P_{shield})$.

Based on the parameters defined for $P_{safe}$ and $P_{shield}$, we show that no disturbance can enter the safe zone when executing a straight task. It is sufficient to show that this cannot happen for a single closed-loop step.  As $\triangle t$ is around $200$ milliseconds, we assume  any lateral error is negligible. 

\begin{theorem}
    Let $\mathcal{O} = \dots O(t-2) O(t-1) O(t) $ be a time series of sets of observations for each closed-loop step during an infinite run.  Then for any set of observations $O(t)$ during execution of a straight task it is guaranteed that a disturbance $D^+$ cannot be witnessed in the safe zone $P_{safe}$.     
\end{theorem}

\begin{proof}
    Assume that during execution of a straight task a disturbance is witnessed in the safe zone, hence $D^+ \in O(t)$ and  $D^+ \in P_{safe}$. Then this means that in the previous time step $D^+ \in O(t-1)$ and $D^+ \not \in P_{shield}$.  As the lateral dimension of $P_{shield}$ is equal to the width of $P_{safe}$, and the lateral error is assumed zero, it is guaranteed that $D_y^+ \in P_{shield}$ for set $O(t-1)$. Hence the inequality $D_x^+ > d_{safe} + v \triangle t + \epsilon$ must hold by the definition of $P_{shield}$ in Equation \ref{eq:shield}. As the robot drives straight distance $v \triangle t + \epsilon$ between sets of observations,  we therefore have the inequality $D_x^+ - v \triangle t + \epsilon > d_{safe}$ for the set of observations $O(t)$.  By the definition of $P_{safe}$ in Equation \ref{eq:safe}, this means $D^+ \in O(t)$ and $D^+ \not \in P_{safe}$ which contradicts our assumption.  Hence a disturbance cannot be witnessed in the safe zone during a straight task. 
\end{proof}

In practice, this result implies that the longitudinal dimension of $P_{shield}$ in our model should scale in proportion to the velocity of the robot.  The error term $\epsilon$ is impossible to predict, so the most sensible strategy is to  over-approximate a positive error using a fixed tolerance $v \triangle t + tol$.  If the error term $\epsilon$ is negative, then it is guaranteed that a disturbance will be witnessed in $P_{shield}$. If it is positive, then as long as $tol$ is at least equal to the error term, it is also guaranteed that a disturbance will be witnessed in $P_{shield}$. The result also implies that the width of $P_{shield}$ and $P_{look}$ should be equal to $d_{safe}$ to guarantee, under the stated assumptions, that the robot is always safe.

\bibliographystyle{ACM-Reference-Format}
\bibliography{references}

\end{document}